\documentclass{article}

\usepackage{thmtools}
\usepackage{thm-restate}

\newcommand{\cS}{\mathcal{S}}

\newcommand{\bE}{\mathbb{E}}

\newcommand{\bI}{\mathbb{I}}

\newcommand{\bP}{\mathbb{P}}

\usepackage{makecell}

\usepackage{amsmath,amsfonts,bm}

\def\eqref#1{equation~\ref{#1}}

\def\1{\bm{1}}

\DeclareMathAlphabet{\mathsfit}{\encodingdefault}{\sfdefault}{m}{sl}
\SetMathAlphabet{\mathsfit}{bold}{\encodingdefault}{\sfdefault}{bx}{n}

\usepackage{microtype}
\usepackage{graphicx}
\usepackage{subfigure}
\usepackage{booktabs} 

\usepackage{hyperref}

\usepackage[accepted]{icml2023}

\usepackage{amsmath}
\usepackage{amssymb}
\usepackage{mathtools}
\usepackage{amsthm}
\usepackage{multicol}
\usepackage{multirow}
\usepackage{booktabs}
\usepackage{wrapfig}

\theoremstyle{plain}
\newtheorem{theorem}{Theorem}[section]

\newtheorem{lemma}[theorem]{Lemma}

\theoremstyle{definition}
\newtheorem{definition}[theorem]{Definition}

\theoremstyle{remark}

\newtheorem{example}[theorem]{Example}

\usepackage[capitalize,noabbrev]{cleveref}

\usepackage{algorithm}
\usepackage{algorithmic}

\usepackage{arydshln}
\usepackage[textsize=tiny]{todonotes}

\icmltitlerunning{On Uni-Modal Feature Learning in Supervised Multi-Modal Learning}

\begin{document}

\twocolumn[
\icmltitle{On Uni-Modal Feature Learning in Supervised Multi-Modal Learning}



\icmlsetsymbol{equal}{*}


\begin{icmlauthorlist}
\icmlauthor{Chenzhuang Du}{equal,iiis}
\icmlauthor{Jiaye Teng}{equal,iiis}
\icmlauthor{Tingle Li}{ucb}
\icmlauthor{Yichen Liu}{iiis}
\icmlauthor{Tianyuan Yuan}{iiis}
\icmlauthor{Yue Wang}{mit}
\icmlauthor{Yang Yuan}{iiis,ailab,qizhi}
\icmlauthor{Hang Zhao}{iiis,ailab,qizhi}
\end{icmlauthorlist}

\icmlaffiliation{iiis}{IIIS, Tsinghua University}
\icmlaffiliation{qizhi}{Shanghai Qi Zhi Institute}
\icmlaffiliation{mit}{Massachusetts Institute of Technology}
\icmlaffiliation{ucb}{University of California, Berkeley}
\icmlaffiliation{ailab}{Shanghai Artificial Intelligence Laboratory}

\icmlcorrespondingauthor{Chenzhuang Du}{ducz20@mails.tsinghua.edu.cn}
\icmlcorrespondingauthor{Hang Zhao}{hangzhao@mail.tsinghua.edu.cn}

\icmlkeywords{Machine Learning, ICML}

\vskip 0.3in
]



\printAffiliationsAndNotice{\icmlEqualContribution} 

\begin{abstract}
We abstract the features~(\textit{i.e.} learned representations) of multi-modal data into 1)~\emph{uni-modal features}, which can be learned from uni-modal training, and 2) \emph{paired features}, which can \emph{only} be learned from cross-modal interactions.
Multi-modal models are expected to benefit from cross-modal interactions on the basis of ensuring uni-modal feature learning.
However, recent supervised multi-modal late-fusion training approaches still suffer from insufficient learning of uni-modal features on each modality. 
\emph{We prove that this phenomenon does hurt the model's generalization ability}.
To this end, we propose to choose a targeted late-fusion learning method for the given supervised multi-modal task from \textbf{U}ni-\textbf{M}odal \textbf{E}nsemble~(UME) and the proposed \textbf{U}ni-\textbf{M}odal \textbf{T}eacher~(UMT), according to the distribution of uni-modal and paired features. 
We demonstrate that, under a simple guiding strategy, we can achieve comparable results to other complex late-fusion or intermediate-fusion methods on various multi-modal datasets, including VGG-Sound, Kinetics-400, UCF101, and ModelNet40.
\end{abstract}

\section{Introduction}
\begin{figure*}[ht]
    \centering
    \includegraphics[width=0.9\linewidth]{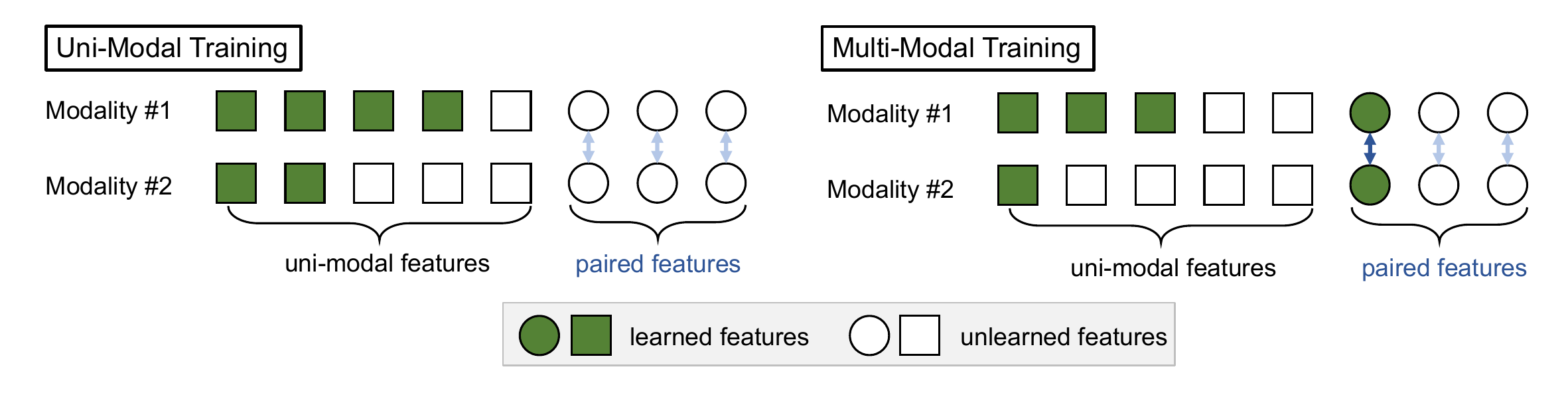}
    \caption{\textbf{Overview of Modality Laziness}. Although multi-modal joint training provides the opportunity for cross-modal interaction to learn paired features, the model easily saturates and ignores the uni-modal features that are hard to learn but also important to generalization. Recent multi-modal methods still suffer from this problem~\cite{peng2022balanced}.
    }
    \label{fig:overview_laziness}
\end{figure*}

Multi-modal signals, \textit{e.g.}, vision, sound, text, are ubiquitous in our daily life, allowing us to perceive the world through multiple sensory systems. Inspired by the crucial role that multi-modal interactions play in human perception and decision \citep{smith2005development}, substantial efforts have been made to build effective and reliable computational multi-modal systems in fields like multimedia computing \citep{wang2020makes, xiao2020audiovisual}, representation learning \citep{radford2021learning} and robotics \citep{chen2020soundspaces}.

\citet{liang2022multiviz} analyzes multi-modal model behavior by studying the uni-modal importance, cross-modal interactions and so on.
In this paper, according to how the features~(\textit{i.e.} learned representations) of multi-modal data can be learned in supervised learning, we abstract them into two categories: (1) \emph{uni-modal features}, which can be learned from uni-modal training, and (2) \emph{paired features}, which can \emph{only} be learned from cross-modal interactions.
Ideally, we hope that multi-modal models can learn paired features through cross-modal interactions on the basis of ensuring that enough uni-modal features are learned. 

However, recent supervised multi-modal late-fusion training methods still suffer from learning insufficient uni-modal features of each modality in tasks where uni-modal priors are meaningful~\footnote{Uni-modal prior here means that we get predictions only according to one modality in multi-modal tasks.}~\citep{peng2022balanced}. Specifically, in linear probing, the encoders from multi-modal learning perform worse than those from uni-modal learning.
We term this phenomenon as \emph{Modality Laziness} and illustrate that in Figure~\ref{fig:overview_laziness}. 
In this paper, we prove that \textbf{it does hurt the generalization ability of the model}, especially when uni-modal features are dominant in given task.

Besides the laziness problem, another shortcoming of recent late-fusion approaches is complex to implement.
For example, G-Blending~\citep{wang2020makes}  needs an extra split of data to estimate the overfitting-to-generalization ratio to re-weight the losses and then re-train the model again and again. OGM-GE~\cite{peng2022balanced}, which dynamically adjusts the gradients of different modalities during training, needs to tune too many hyper-parameters, including the start and end epoch of the gradient modulation, an ``alpha" used to calculate the coefficients for the modulation and whether adaptive Gaussian noise Enhancement (GE) is needed. The more complicated thing is that these hyper-parameters need to be re-tuned on new datasets~\footnote{\url{https://github.com/GeWu-Lab/OGM-GE_CVPR2022}}.

To this end, more simple and effective methods are urgently needed.
We pay attention to the learning of uni-modal features and propose to choose targeted late-fusion training method for the given task from \textbf{U}ni-\textbf{M}odal \textbf{E}nsemble~(UME) and the proposed  \textbf{U}ni-\textbf{M}odal \textbf{T}eacher~(UMT) according to the distribution of uni-modal and paired features:
\begin{itemize}
    \item If both uni-modal and paired features are essential, UMT is effective, which helps multi-modal models better learn uni-modal features via uni-modal distillation and also preserves cross-modal interactions; 
    \item If paired features are not important and both modalities have strong uni-modal features, UME is properer, which directly combines outputs of uni-modal models and almost avoids cross-modal interactions that may lead to Modality Laziness.
\end{itemize}
We also provide an empirical trick to decide which one to use for a given task. 
Under this guidance, we achieve comparable results to other complex late-fusion or intermediate-fusion methods on multiple multi-modal datasets, including VGG-Sound~\citep{chen2020vggsound}, Kinetics-400~\citep{kay2017kinetics}, UCF101~\citep{soomro2012ucf101} and ModelNet40~\citep{wu2022characterizing}.

\section{Related Work}
\textbf{Multi-modal training approaches} aim to train a multi-modal model by using all available modalities~\citep{liang2021multibench}. Common supervised multi-modal tasks include audio-visual classification~\citep{peng2022balanced,xiao2020audiovisual, panda2021adamml}, action recognition~\citep{wang2020makes, panda2021adamml}, visual question answering~\citep{agrawal2018don}, RGB-D segmentation~\citep{park2017rdfnet, hu2019acnet, esanet2020} and so on. There are several different fusion methods, including early/middle fusion~\citep{esanet2020,nagrani2021attention, wu2022characterizing} and late fusion~\citep{wang2020makes, peng2022balanced,fayek2020large}. 
In this paper, we mainly improve the late-fusion methods following~\citet{wang2020makes}, which is convenient and straightforward to evaluate the learning of uni-modal features. 

\textbf{Multi-modal learning theory.}
The research on multi-modal learning theory is still at an early age.
A line of work focuses on understanding multi-view tasks \citep{amini2009learning,xu2013survey,arora2016stochastic,allen2020towards}, and our assumption on the data structure partially stems from \citet{allen2020towards}. 
\citet{huang2021makes} explains why multi-modal learning is potentially better than uni-modal learning and \citet{huang2022modality} explains why failure exists in multi-modal learning. 
Several works~\citep{hessel2020does,liang2022multiviz} also have analyzed the cross-modal interactions. 
Our paper investigates the different types of features in multi-modal data and provides solutions for the weakness of multi-modal learning.

\textbf{Knowledge distillation} was introduced to compress the knowledge from an ensemble into a smaller and faster model but still preserve competitive generalization power~\citep{bucilu2006model,hinton2015distilling,tian2019contrastive,gou2021knowledge, allen2020towards}. 
In this paper, we propose Uni-Modal Teacher, which leverages uni-modal distillation for joint training to help the learning of uni-modal features, without involving cross-modal knowledge distillation~\citep{pham2019found, gupta2016cross, tan2020vokenization,garcia2018modality,luo2018graph}.

\begin{table*}[t]
	\centering
	\renewcommand{\arraystretch}{1.1}
	\caption{Top 1 test accuracy (in \%) of linear evaluation on encoders from various multi-modal late-fusion training methods and uni-modal training on VGG-Sound and UCF101.}
	\vspace{5pt}
	\begin{tabular}{c|c c |c c}
		\toprule
		\multirow{2}*{\textbf{Method}}& \multicolumn{2}{c|}{\textbf{VGG-Sound}} & \multicolumn{2}{c}{\textbf{UCF101}} \\ \cline{2-5} & RGB Encoder & Audio Encoder & RGB Encoder & Opt-Flow Encoder \\
		\midrule
		Linear-Fusion & 15.56 & 43.44  & 75.66 & 48.08\\
		MLP-Fusion & 14.52 & 40.01  & 75.65 & 51.89\\
		Attention-Fusion & 13.31 & 43.97  & 74.84 & 7.72\\
		G-Blending & 17.69 & 43.90  & 74.91 & 44.49\\
		OGM-GE & 15.60 & 41.95 & 73.54 & 65.03\\
		\midrule
		Uni-Modal Training &  \textbf{23.17} & \textbf{45.15} &\textbf{77.08} &  \textbf{74.99} \\
		
		\bottomrule
	\end{tabular}
	\label{tb:modality_laziness}
\end{table*}

\begin{figure*}[t]
\centering 
    \subfigure[RGB encoder evaluation on VGG-Sound.] { \label{fig:a} 
    \includegraphics[width=0.85\columnwidth]{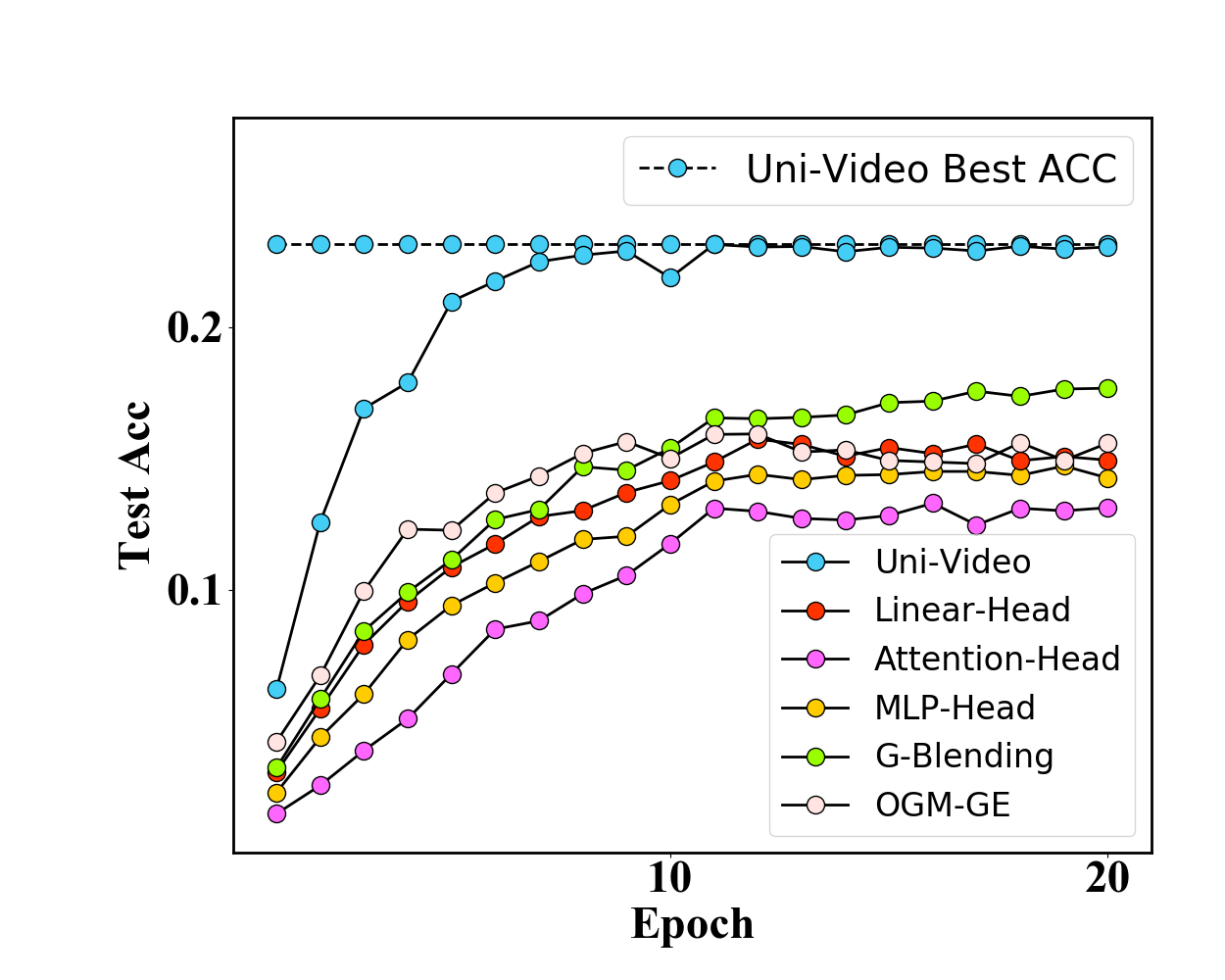} 
    } 
    \subfigure[Optical flow encoder evaluation on UCF101.] { \label{fig:b} 
    \includegraphics[width=0.85\columnwidth]{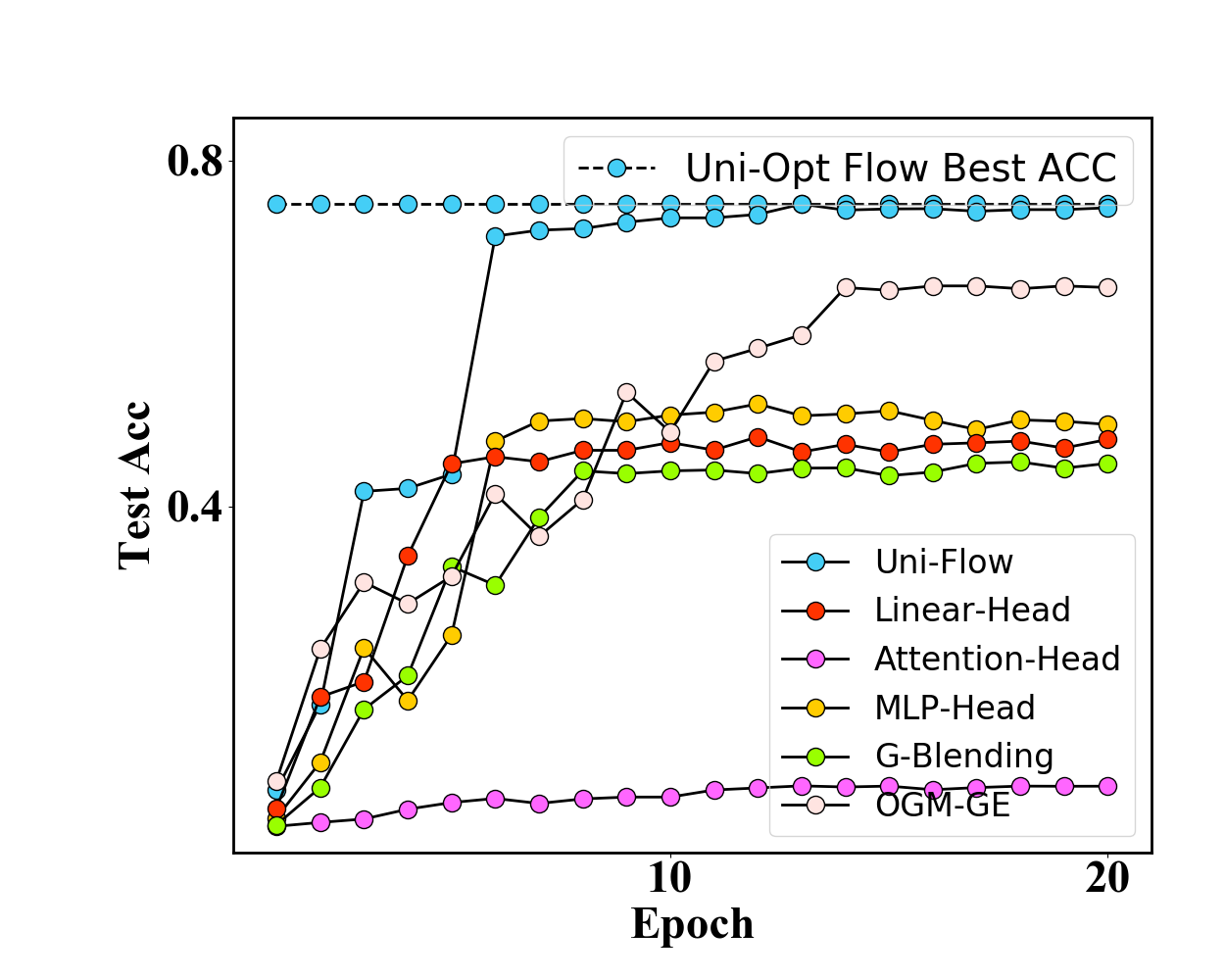} 
    } 
\caption{By building a linear classifier on encoders and checking the top-1 accuracy in the training process, we evaluate the RGB encoders in VGG-Sound and the optical flow encoders in UCF101 from different multi-modal late-fusion methods.} 
\label{fig:laziness_process} 
\end{figure*}

\begin{table*}[ht]
    \centering
    \caption{
    Top-1 test accuracy of multi-modal models with different freedom of cross-modal interactions and uni-modal models on certain classes of VGG-Sound. 
    }
    \vspace{5pt}
    \begin{tabular}{c|cccccccccc|c}
    \toprule
      Class ID   & 164  & 303 &  33& 255 & 91 & 4 & 152 & 127 & 68 & 155  & mean acc  \\
    \midrule   
    Uni-RGB&3  &2 &4&3&4&12 &2 &0  &15  & 5 &   5     \\
    Uni-Audio&30  & 7 &34&10&43& 50 & 18 & 0 &53  & 32 &    27.7       \\
    \midrule
    Avg Pred & 37 & 10& 37 & 7 & 28 & 63 & 21 & 0 & 51& 30& 28.4\\
    Linear Clf & 35 & 15 & 33 & \textbf{27}& \textbf{60} & 65 & 26 & 2 & 53& \textbf{49}& 36.5\\
    Naive Fusion& \textbf{43}   & \textbf{18} & \textbf{48} & 22& 55& \textbf{67} & \textbf{26} & \textbf{4}& \textbf{72} & 40 & \textbf{39.5}      \\
    \bottomrule
    \end{tabular}
    \label{tb:class}
\end{table*}

\section{Analysis, Learning Guidance and Theoretical Guarantee}
In this section, we first illustrate Modality Laziness still exists in recent multi-modal late-fusion learning. 
And then we investigate how multi-modal models benefit from joint training. 
Based on this, we propose to choose the appropriate learning method for the given multi-modal task according to its characteristic.
Finally, we provide a theoretical analysis of Modality Laziness and justification for our solution.

\textbf{Discussion.}
The importance of uni-modal prior varies across different multi-modal tasks.
In tasks like video classification~\citep{chen2020vggsound} and action recognition~\citep{feichtenhofer2016convolutional,wang2020makes}, uni-modal models can achieve good performance alone, suggesting that uni-modal priors in these settings are essential. Visual question and answering (VQA)~\citep{agrawal2018don} is a counter example. Specifically, the same image with different text questions may have totally different labels, making it pointless to check its uni-modal accuracy.
In tasks where uni-modal priors are essential, recent multi-modal training methods~\cite{wu2022characterizing,peng2022balanced} still suffer from insufficient learning of uni-modal features.
In this paper, we focus on analyzing and improving the performance of late-fusion methods in these tasks.

\subsection{Modality Laziness in Multi-modal Training}
\label{sec3_1: observation_of_laziness}


\textbf{In multi-modal late-fusion learning}, each modality is encoded by its corresponding encoder and then a fusion module is applied on top of them to produce outputs. 
In late-fusion learning, by building a classifier on the trained and frozen encoder~\citep{chen2020simple}, we can assess the learned representations of the encoder~(\textit{i.e.} linear probing). 
We find that recent methods, including G-Blending~\cite{wang2020makes} and OGM-GE~\cite{peng2022balanced}, still suffer from insufficient learning of uni-modal features, namely Modality Laziness:
\begin{itemize}
    \item As Table~\ref{tb:modality_laziness} shows, all encoders from multi-modal joint training are worse than those from uni-modal training, especially the RGB encoder in VGG-Sound and optical flow encoder in UCF101. No matter which optimizer is used~(Appendix~\ref{app_opt}).
    \item As Figure~\ref{fig:laziness_process} shows, throughout the training process, the two encoders mentioned above not only cannot achieve comparable performance to their uni-modal counterparts but are far worse than them. 
\end{itemize}
\subsection{How does a Multi-modal Model Benefit from Multi-modal Training?}
\label{sec3_4: paired}

In Sec~\ref{sec3_1: observation_of_laziness}, we empirically show that recent late-fusion methods still suffer from insufficient learning of uni-modal features. 
A straightforward solution is to train the uni-modal models individually and then combine their predictions to give the final prediction.
However, it raises another question: 
\emph{How does a multi-modal model benefit from multi-modal joint training?} 
We hypothesize cross-modal interaction plays a role and investigate several models with different freedom of cross-modal interactions on VGG-Sound, including 
1) directly average the uni-modal models' predictions, which has few cross-modal interactions; 
2) train a multi-modal linear classifier on top of uni-modal pre-trained but frozen encoders, where modalities can interact with each other through the linear layer;
3) naive fusion or naive multi-modal learning: end-to-end late-fusion learning from scratch without carefully designed tricks, where the modalities can interact more than the two models above.

As Table~\ref{tb:class} shows, in certain classes of VGG-Sound, the accuracy of naive fusion exceeds the sum of the accuracy of the two uni-modal models.
Besides, we find that naive fusion training, which owns maximum freedom of cross-modal interactions among these models, gets the best mean accuracy across these classes. 
And averaging the predictions of uni-modal models, which owns minimum freedom of cross-modal interactions among these models, gets the worse mean accuracy across these classes.
These results suggest that joint training enables the model to learn representations beyond uni-modal features, which we term as \emph{paired features}. We give the formal mathematical definitions of uni-modal features and paired features in Sec~\ref{sec3_2: theory} and offer more explanations on paired features in Appendix~\ref{app_paired}.

\begin{algorithm}[t]
  \caption{Uni-Modal Teacher (UMT)}
  \label{alg:umt}
\begin{algorithmic}
  \STATE {\bfseries Input:} Uni-modal supervised pre-trained models $F_{pretrain}^{m_1}, F_{pretrain}^{m_2}$, random initialized late-fusion multi-modal model $F^{mm}$, iteration number $N$, loss weight $\lambda_{task}, \lambda_{distill}$.
  \FOR{$0$ {\bfseries to} $N$}
  \STATE Sample multi-modal data $\{X^{m_1}, X^{m_2}, Y\} \sim \mathcal{D}$.
  \STATE Compute uni-modal pre-trained features $f_{pre}^{m_1}, f_{pre}^{m_2}$ of the data by $F_{pretrain}^{m_1}, F_{pretrain}^{m_2}$.
  \STATE Compute the prediction and features $\hat{Y}, f^{m_1}, f^{m_2}$ from multi-modal model.
  \STATE Compute the losses between $\hat{Y}, f^{m_1}, f^{m_2}$ and $Y,  f_{pre}^{m_1}, f_{pre}^{m_2}$ and multiply by the $\lambda_{task}, \lambda_{distill}, \lambda_{distill}$, respectively.
  \STATE Update the multi-modal model by SGD or its variant.
  \ENDFOR
  \STATE {\bfseries Return:} A multi-modal model trained by UMT.
\end{algorithmic}
\end{algorithm}

\subsection{Guidance on Multi-modal Learning}
\label{sec3_3: pushing_force}

In Appendix~\ref{app_class}, we analyze the role of cross-modal interaction on more datasets and find that whether it brings benefits or disadvantages highly depends on the task itself, which motivates us to propose that we should choose proper learning method from proposed \textbf{U}ni-\textbf{M}odal \textbf{T}eacher~(UMT) and \textbf{U}ni-\textbf{M}odal \textbf{E}nsemble~(UME) for the given task.

\textbf{UMT.} 
\textbf{U}ni-\textbf{M}odal \textbf{T}eacher (UMT) is proposed for late-fusion training. 
It distills the pre-trained uni-modal features to the corresponding parts of multi-modal late-fusion models.
The framework of UMT is shown in Algorithm~\ref{alg:umt} and Figure~\ref{fig:naive_fusion}. 
More details about UMT can be found in Appendix~\ref{app_umt_details}. 
There are several important differences between UMT and the approach in \citet{wang2020multimodal}. First, the distillation in UMT happens in feature-level but not soft label level. Second, the training of uni-modal models in UMT does not use any additional data compared to the training of the multi-modal model. And our motivation is to enable the multi-modal model to better learn the uni-modal features in the current dataset, rather than introduce additional information to the multi-modal model.

\textbf{UME.} 
\textbf{U}ni-\textbf{M}odal \textbf{E}nsemble~(UME) aims to avoids insufficient learning of uni-modal features by combining predictions of uni-modal models.
Firstly, we can train uni-modal models independently.
Then, we can give final output by weighting the predictions of uni-modal models.


\textbf{An empirical trick to decide which method to use.} 
We can train a multi-modal linear classifier on uni-modal pre-trained encoders and different modalities can interact with each other in the linear layer. 
And then, we compare that with averaging predictions of uni-modal models, which has few cross-modal interactions:
\begin{itemize}
    \item If the performance of the classifier is better, it means we can benefit from cross-modal interactions in this task and we can choose UMT, where cross-modal interactions are preserved while guaranteeing improved learning of uni-modal features;
    \item Otherwise, the cross-modal interactions does more harm than good in the given task, and we can choose UME, which almost avoids cross-modal interactions. 
\end{itemize}

Noting that in UMT and UME, we use the same backbone in uni-modal and multi-modal models for a specified modality. 
And in the following subsection, we give theoretical justification for our solution.

\newcommand{\Combine}{multi-modal training approaches}

\newcommand{\Vote}{uni-modal ensemble}
\newcommand{\Uni}{uni-modal approaches}
\newcommand{\singFeature}{uni-modal feature}
\newcommand{\pairFeature}{paired feature}
\newcommand{\modalOne}{{x^{m_1}}}
\newcommand{\modalTwo}{{x^{m_2}}}

\newcommand{\pFi}{p(f_i)}
\newcommand{\eFi}{\epsilon(f_i)}
\newcommand{\pHi}{p(h_j)}
\newcommand{\eHi}{\epsilon(h_j)}

\newcommand{\kpa}{{k_{\text{pa}}}}
\newcommand{\kmone}{{k_{\text{m1}}}}
\newcommand{\kmtwo}{{k_{\text{m2}}}}
\newcommand{\bmone}{{b_{\text{m1}}}}
\newcommand{\bmtwo}{{b_{\text{m2}}}}

\subsection{Theoretical Characterization and Justification}
\label{sec3_2: theory}

\begin{figure*}[t]
    \centering
    \includegraphics[width=0.8\linewidth]{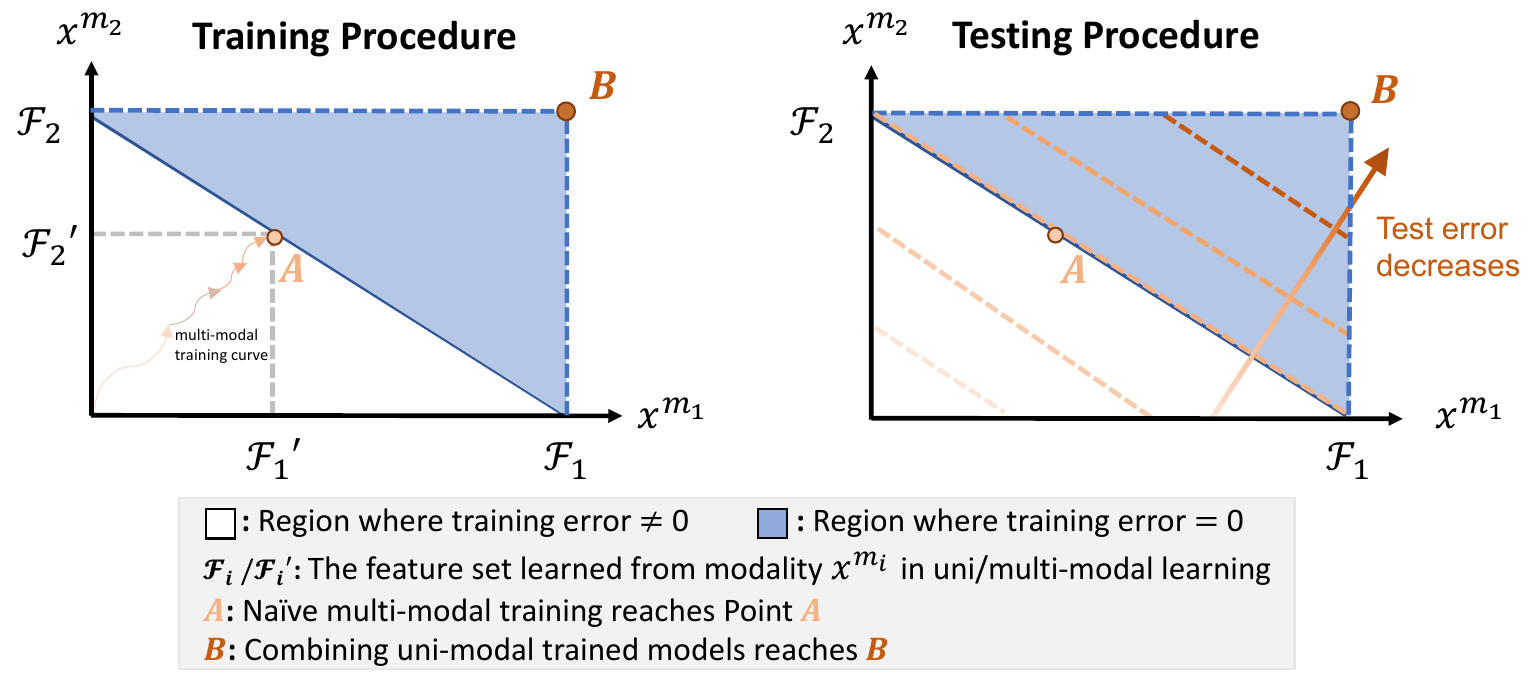}
    \caption{An illustration of the feature learning results of uni-modal training and multi-modal training without considering paired features.
    In uni-modal training, modality $x^{m_i}$ learns feature set $\mathcal{F}_i$. However, naive joint training learns less features of each modality than uni-modal training when getting zero training error~(namely $\mathcal{F}_i^{'}$).
    Uncontroversially, combining predictions of individual trained uni-modal models~(B) outperforms naive joint training~(A).}
    \label{fig:starve}
\end{figure*}

In this subsection, we characterize Modality Laziness of Sec~\ref{sec3_1: observation_of_laziness} from a feature learning perspective and prove it does hurt the generalization of the model. And then, we give justification for the learning guidance proposed in Sec~\ref{sec3_3: pushing_force}.

Before diving into the technical details, we first provide some intuition behind the proof. 
Our goal is to show that how Modality Laziness happens in multi-modal joint training, and we refer to Figure~\ref{fig:starve} as an illustration. 
Here we omit the effect of paired features for easier to understand the intuition.
During the naive multi-modal training process, learning those easy-to-learn features suffices to reach zero training error (point A in Figure~\ref{fig:starve}).
However, the model is  under-trained at point A, and the zero-training-error region stops us from further training.
As a comparison, uni-modal models can learn more features and achieve point B, outperforming point A.

We next give Modality Laziness a theoretical explanation under a simple but effective regime.
We mainly consider cases with two modalities $\modalOne$ and $\modalTwo$, and \emph{similar techniques can be directly generalized to the cases with more modalities}.

\textbf{Data distribution.}
We formalize the distribution of the multi-modal features.
Specifically, we abstract the features into \singFeature s\ and \pairFeature s to describe the core differences between uni-modal training and multi-modal joint training. And we give these two types of features mathematical definitions in Definition~\ref{def:self-standing} and Definition~\ref{def:paired}, respectively.
And we consider the binary classification regime where the label $y$ has a uniform distribution over $\{-1, 1\}$ without loss of generality.
Such simplification is self-contained to describe the differences between \singFeature s\ and \pairFeature s.


\begin{definition}[Uni-modal features, \emph{which can be learned from uni-modal training}]
\label{def:self-standing}
The $i$-th \singFeature\ ($f_i(\modalOne)$) in modality $\modalOne$ is generated as\footnote{We simplify ``with probability'' as ``w.p.''}:
\begin{equation*}
    \begin{split}
        \text{w.p. $p(f_i)$, \ }& y f_i(\modalOne) > 0; \\
        \text{w.p. $1- p(f_i)-\epsilon(f_i)$, \ }& y f_i(\modalOne) = 0;\\
        \text{w.p. $\epsilon(f_i)$, \ }& y f_i(\modalOne) < 0.
    \end{split}
\end{equation*}

The $i$-th \singFeature\ ($g_i(\modalTwo)$) in modality $\modalTwo$ is similarly generated with parameters $p(g_i)$ and $\epsilon(g_i)$.
\end{definition}

\begin{definition}[Paired features, \emph{which can only be learned from cross-modal interaction}]
\label{def:paired}
The $j$-th \pairFeature\footnote{We abuse the notation $h$ to simplify the notations where $h(\modalOne)$ and $h(\modalTwo)$ can have different forms.} $h_j$ is generated as:
\begin{equation*}
    \begin{split}
        \text{w.p. $\pHi$, \ }& y h_j(\modalOne)h_j(\modalTwo) > 0; \\
        \text{w.p. $1- \pHi - \eHi$, \ }& y h_j(\modalOne)h_j(\modalTwo) = 0;\\
        \text{w.p. $\eHi$, \ }& y h_j(\modalOne)h_j(\modalTwo) < 0.
    \end{split}
\end{equation*}
\end{definition}

When the context is clear, we abuse the notation $r_i$ to represent either $f_i$ (\singFeature\ in modality $\modalOne$), $g_i$ (\singFeature\ in modality $\modalTwo$), or $h_i$ (\pairFeature).
We name $p(r_i)$ as the \emph{predicting probability} of feature $r_i$.
When $r_i$ is present (meaning that $r_i \neq 0$), we use $\bI(r_i>0) - \bI(r_i<0)$ to predict $y$.
Otherwise ($r_i=0$), we random guess $y$ uniformly over $\{-1, 1\}$.
To simplify the discussion, we always assume $\eFi = \pFi /c$, where $c>1$ is a fixed constant.
For the ease of notations, we define the empty feature in Definition~\ref{def:empty}.

\begin{definition}[Empty Feature]
\label{def:empty}
Empty feature $e_i$ is a kind of \singFeature\ (or \pairFeature) with $p(e_i) = \epsilon(e_i) = 0$.
\end{definition}

\textbf{Evaluation procedure.} 
When the context is clear, we abuse $r_i$ to denote the learned features.
For each data point, we random guess $\hat{y}$ on $\{-1, 1\}$ uniformly when $\sum_i \bI(r_i>0) = \sum_i \bI(r_i<0)$.
Otherwise, we predict the label by $\hat{y} = 2\bI(\sum_i \bI(r_i >0) > \sum_i \bI(r_i<0)) - 1$. 
We define the error as $ \sum_i \bI(y r_i<0) - \sum_i \bI(y r_i>0) $.

\textbf{Training procedure.}
(a.) \emph{multi-modal joint training}, which directly train the model using both modality $\modalOne$ and modality $\modalTwo$;
(b.) \emph{uni-modal ensemble}, which firstly train the features via independent training~($\modalOne$ and $\modalTwo$ separately), and then combine the $\modalOne$-learned features and $\modalTwo$-learned features.

During the training process, we first initialize all the features with empty features $e_i$ to imitate random initialization.
The models then learn the features in descending order of predicting probability, meaning that the powerful features (with large predicting probability) are learned first\footnote{Recent works have demonstrated that neural networks indeed prefer easy-to-learn features \citep{shah2020pitfalls,pezeshki2020gradient}.}.
Our goal is to minimize the training error to zero\footnote{We always assume that the training error can be minimized to zero.}.

We now state our main theorem in Theorem~\ref{thm:combineoverfitting}, demonstrating naive joint training learn fewer uni-modal features compared to uni-modal training, which hurts the model's generalization.


\begin{restatable}{theorem}{Combineoverfitting}
\label{thm:combineoverfitting}
In \Vote, assume that the training procedure learns $\bmone$ features in modality $\modalOne$ and learns $\bmtwo$ features in modality $\modalTwo$.
We order the probability of \singFeature s\ (both $\modalOne\ $ and $\modalTwo$) in decreasing order of predicting probability $p$, namely, $p_{[1]}, p_{[2]}, \dots$.
In \Combine, assume that the training procedure learns $\kmone$ \singFeature s\ in modality $\modalOne$, learns $\kmtwo$ \singFeature s\ in modality $\modalTwo$, and learns $\kpa$ \pairFeature s\ with predicting probability $p(h_1), \dots, p(h_{\kpa})$.

We provide three types of laziness:
\begin{enumerate}
    \item[(a. )] \textbf{Quantity Laziness}: $\kmone+ \kmtwo +\kpa\leq \min\{\bmone, \bmtwo\}$.
    \item[(b. )] \textbf{Uni-modal Laziness}: Each modality in \Combine\ performs worse than uni-modal training.
    \item[(c. )] \textbf{Performance Laziness}:
    Consider a new testing point, then for every $\delta>0$, if the following inequality holds:
    $$\sum_{i \in [\kpa]} p(h_i) \leq \sum_{i \in [\bmone+1, \bmone+\bmtwo]} p_{[i]} + \Delta(\delta),
    $$
    where $\Delta(\delta) = \sqrt{8 (\kpa + \bmone - \kmone + \bmtwo - \kmtwo) \log(1/\delta)}$, then with probability\footnote{The probability is taken over the randomness of the testing point} at least $1-\delta$, \Vote\ outperform \Combine\ concerning the loss on the testing point with probability.
\end{enumerate}
\end{restatable}

\begin{figure*}[t]
\begin{minipage}[t]{\textwidth}
\begin{minipage}[t]{0.5\textwidth}
\makeatletter\def\@captype{table}
\centering
\caption{Top-1 test accuracy of averaging uni-modal predictions and multi-modal classifier trained on uni-modal pre-trained encoders.}
\vspace{5pt}
\begin{tabular}{l|cc}
    \toprule
    \textbf{Dataset}  &\textbf{MM Clf}   & \textbf{Avg Preds}  \\
    \midrule
    VGG-Sound & \textbf{51.0} & 46.1   \\
    Kinetics-400    & \textbf{76.4} & 74.8 \\
    UCF101 & 84.4   &  \textbf{86.8} \\
    ModelNet40 & 91.7  & \textbf{91.9}\\
    \bottomrule
\end{tabular}
\label{tb:avg_linearclf}
\end{minipage}
\begin{minipage}[t]{0.5\textwidth}
\makeatletter\def\@captype{table}
\centering
\caption{Top-1 test accuracy of uni-modal classifiers from the multi-modal linear classifier trained on uni-modal pre-trained encoders in UCF101. This evaluation method borrows from \citet{peng2022balanced}.}
	\vspace{5pt}
	\begin{tabular}{c| cc}
	\toprule
	\textbf{Model}  & \textbf{RGB} & \textbf{Opt-flow} \\
	\midrule
	    Uni-Clf from MM Clf & 68.2&  52.9\\
	    Uni-Modal Model  & \textbf{77.1} &\textbf{75.0}\\
	   
	 \bottomrule
\end{tabular}
\label{tb:uni_from_mm}
\end{minipage}
\end{minipage}
\end{figure*}

\begin{figure*}[t]
\begin{minipage}[t]{\textwidth}
\begin{minipage}[t]{0.5\textwidth}
\makeatletter\def\@captype{table}
\centering
\caption{Results of different late-fusion methods~(top-1 test accuracy). * means the result comes from its original paper}
\vspace{1pt}
\begin{tabular}{l|cc}
    \toprule
    \textbf{Method}  &  \textbf{VGG-Sound}  & \textbf{Kinetics-400}\\
    \midrule
    Linear-Head & 49.5& 74.3  \\
    MLP-Head    & 44.8& 74.8 \\
    Atten-Head &  49.8  & 74.1 \\
    Aux-CELoss & 49.9 &73.2\\
    G-Blending & 50.4 &75.8* \\
    OGM-GE & 50.6* & 74.5 \\
    \midrule
    UMT~(ours) & \textbf{53.5} & \textbf{76.8}\\
    \bottomrule
\end{tabular}
\label{tb:umt_classification}
\end{minipage}
\begin{minipage}[t]{0.43\textwidth}
\makeatletter\def\@captype{table}
\centering
\caption{Top-1 test accuracy of UMT and Audio-Visual SlowFast~\citep{xiao2020audiovisual} on Kinetics-400. AVSlowFast is an representative intermediate fusion method. }
\vspace{5pt}
\begin{tabular}{l|c|c}
    \toprule
    \textbf{Method}  & \textbf{RGB Encoder}   & \textbf{Acc}      \\
    \midrule
    AVSlowFast & SlowFast-50& 77.0\\
    UMT~(ours) &   SlowFast-50&     \textbf{78.1}\\
    
    \midrule
    AVSlowFast &   SlowFast-101 & 78.8    \\
    UMT~(ours) &   SlowFast-101 &   \textbf{79.4} \\

    \bottomrule
\end{tabular}
\label{tb:UMT_AVSlowFast}
\end{minipage}
\end{minipage}
\end{figure*}

In theorem~\ref{thm:combineoverfitting}, we describe three notations of laziness problem:
\textbf{Quantity Laziness} indicates that the number of features learned in naive multi-modal training is less than uni-modal training.
\textbf{Uni-modal Laziness} shows encoders from multi-modal training perform worse than from uni-modal training because of Quantity Laziness, which fits the experimental results in sec3.1.
\textbf{Performance Laziness} compares the performance of multi-modal joint training approaches with Uni-Modal Ensemble, demonstrating that when uni-modal features dominate, combining uni-modal predictions is effective.
We defer the complete proof to Appendix \ref{append:proof} and generalize that to more modalities~(Appendix \ref{generalize_multi_modal}).
We give a concrete example in Appendix B.3 to better illustrate Theorem~\ref{thm:combineoverfitting}.


We next prove that UMT proposed in Sec~\ref{sec3_3: pushing_force} indeed helps uni-modal feature learning and can also learn some easy-to-learn paired features in Theorem~\ref{thm:umt} and Appendix~\ref{append:proof_umt}.


\begin{restatable}{theorem}{umt}
\label{thm:umt}
Denote the paired features\ by $h_1, \dots h_L$ with corresponding predicting probability $p(h_1), \dots, p(h_L)$.
Assume that distillation can boost the training priority by $p^0>0$.
If there exists paired features whose predicting probability exceeds the boosting probability $p^0$, namely, the set $\mathcal{S}$ is not empty:
$$
\mathcal{S} = \{h_i: p(h_i) > p^0 \} \neq \phi.
$$
Then UMT helps uni-modal feature learning and can also learn easy-to-learn paired features.
\end{restatable}

\begin{figure*}[t]
\begin{minipage}[t]{\textwidth}
\begin{minipage}[t]{0.5\textwidth}
\makeatletter\def\@captype{table}
\centering
\caption{Top-1 test accuracy of the encoders trained by naive multi-modal training and UMT.}
\vspace{5pt}
\begin{tabular}{c c c c c }
		\toprule
		\multirow{2}*{\textbf{Methods.}} & \multicolumn{2}{c}{\textbf{VGG-Sound}} & \multicolumn{2}{c}{\textbf{Kinetics-400}}\\
		\cline{2-5} & RGB & Audio & RGB & Audio\\
		\midrule
		Uni-Train & 23.2  & 45.2 & 74.1 & \textbf{23.5}\\
	    \cdashline{1-5}
		MM Baseline & 15.9 &  43.4 & 72.9 & 18.3 \\
		UMT  & \textbf{24.4}  & \textbf{45.9} & \textbf{74.6} & 21.6\\
		\bottomrule
\end{tabular}
\label{tb:improve_encoders}

\end{minipage}
\begin{minipage}[t]{0.5\textwidth}
\makeatletter\def\@captype{table}
\centering
\caption{Self-Distillation vs UMT on VGG-Sound. }
\vspace{5pt}
\begin{tabular}{l|c}
    \toprule
    \textbf{Method}     & \textbf{Top-1 Test Acc}      \\
    \midrule
    Baseline &   49.5     \\
    Self-Distill (label) & 49.7              \\
    Self-Distill (feature)  & 49.9           \\
    \midrule
    UMT & \textbf{53.5}\\
    \bottomrule
\end{tabular}

\label{tb:abla_distill}

\end{minipage}
\end{minipage}
\end{figure*}

\begin{figure*}[t]
\begin{minipage}[t]{\textwidth}
\begin{minipage}[t]{0.5\textwidth}
\makeatletter\def\@captype{table}
\centering
\caption{Comparison Uni-Modal Ensemble with other joint training methods on UCF101.}
\vspace{5pt}
\begin{tabular}{l|c}
    \toprule
    \textbf{Method}  &   \textbf{Top-1 Test Acc}\\
    \midrule
    Linear-Head & 82.3  \\
    MLP-Head    & 80.0 \\
    Atten-Head &  74.2\\
    Aux-CELoss & 81.3\\
    G-Blending & 83.0 \\
    OGM-GE & 84.0 \\
    \midrule
    UME~(ours) & \textbf{86.8}\\
    \bottomrule
\end{tabular}
\label{tb:ucf-ume}

\end{minipage}
\begin{minipage}[t]{0.5\textwidth}
\makeatletter\def\@captype{table}
\centering
\caption{Comparison between Uni-Modal Ensemble and balanced multi-modal learning algorithm~\citep{wu2022characterizing} on ModelNet40. * means the result comes from \cite{wu2022characterizing}.}
\vspace{5pt}
\begin{tabular}{l|c}
    \toprule
    \textbf{Method}     & \textbf{Top-1 Test Acc}      \\
    \midrule
    multi-modal~(vanilla) &   90.09 $\pm$ 0.58*     \\
     +RUBi & 90.45 $\pm$ 0.58*             \\
     +random & 91.36 $\pm$ 0.10*             \\
     +guided & 91.37 $\pm$ 0.28*            \\
    \midrule
    UME~(ours) & \textbf{91.92 $\pm$ 0.14}\\
    \bottomrule
\end{tabular}

\label{tb:modelnet_ume}
\end{minipage}
\end{minipage}
\end{figure*}

\section{Experiments}
\label{sec4:exp}
In Sec~\ref{sec3_2: theory}, we justify our method theoretically.
In this section, we firstly introduce the experimental setup and then demonstrate that our simple solution is effective in various multi-modal tasks.

\subsection{Experimental Setup}
We run experiments on four datasets.
\emph{Kinetics-400}~\citep{kay2017kinetics} is a video recognition dataset with 240k videos for training and 19k for validation. We treat the two modalities, RGB and audio, as the inputs. 
\emph{VGG-Sound} \citep{chen2020vggsound} is an audio-visual classification dataset which contains over 200k video clips for 309 different sound classes. 
\emph{UCF101} \citep{soomro2012ucf101} is an action recognition dataset with 101 action categories, including 7k videos for training and 3k for testing. 
\emph{ModelNet40} is a 3D object classification task with 9,483 training
samples and 2,468 test samples. Following \citet{wu2022characterizing}, we treat the front and rear view as two modalities. Due to limited space, we put other experimental details in Appendix~\ref{appendix:dataset} and~\ref{app_hyper}.


\subsection{An empirical trick to decide which learning method to use.}
We train a multi-modal linear classifier on frozen uni-modal pre-trained encoders and compare that with averaging uni-modal predictions.
As Table~\ref{tb:avg_linearclf} shows, in VGG-Sound and Kinetics-400, the classifier is better, meaning cross-modal interaction can benefit the classifier in the two datasets. 
However, in UCF101 and ModelNet40, averaging uni-modal predictions performs well.
To explore why the classifier fails in UCF101, we check the uni-modal classifiers of the newly trained multi-modal linear layer~(details in Appendix~\ref{app_uniclf_from_mmclf}). As Table~\ref{tb:uni_from_mm} shows, they are far worse than the uni-modal models. 
The result shows the \emph{simple linear classifier suffers from serious Modality Laziness in UCF101}, which negatively impact the performance. 
Both modalities in ModelNet40 also have strong uni-modal features and can achieve 89\% accuracy individually. Averaging uni-modal predictions avoids laziness problem and achieves competitive performance.
Based on the above analysis, we perform UMT on VGG-Sound and Kinetics-400, and UME on UCF101 and ModelNet40.

\subsection{UMT is an effective regularizer}
In this subsection, we demonstrate that Uni-Modal Teacher outperforms other multi-modal training methods in VGG-Sound and Kinetics-400.

\textbf{UMT vs Other Late-Fusion Methods.} 
The late-fusion architecture is commonly used for multi-modal classification tasks~\citep{wang2020makes,peng2022balanced}.
In late-fusion architecture, the features are extracted from different modalities by the corresponding encoders, and then the head layer is applied to output predictions.
We compare different heads, including linear layer, MLP, and attention layer. 
In UMT, we use a simple linear layer as the multi-modal head.
We also conduct another experiment, which adds extra uni-modal linear heads to receive the uni-modal features and generating additional losses to joint optimize the model, namely Auxiliary-CEloss.
Auxiliary-CEloss gives all losses equal weights, while G-Blending reweights the losses according to the overfitting-to-generalization-ratio (OGR)~\citep{wang2020makes}. 
OGM-GE~\citep{peng2022balanced} controls the optimization of each modality by online gradient modulation.
As shown in Table~\ref{tb:umt_classification}, UMT outperforms other methods.

\textbf{UMT vs AVSlowFast.} Audio-Visual SlowFast is an representative intermediate fusion method. We compare UMT with AVSlowFast in Kinetics-400. As Table~\ref{tb:UMT_AVSlowFast} shows, under different RGB encoders, UMT consistently exceeds AVSlowFast, although we cannot reproduce their results due to the dynamics of Kinetics-400~(Appendix~\ref{app_uni_model-kinetics}).

\textbf{Ablation Study of UMT.} 
We first evaluate the encoders of UMT by training linear classifiers on them. As Table~\ref{tb:improve_encoders} shows,  UMT makes its encoders stand out. Benefiting from uni-modal distillation, some encoders even outperform their uni-modal counterparts.
We then compare UMT with classic self-distillation methods~(distillation on soft label~\citep{hinton2015distilling} and feature~\citep{romero2014fitnets}).
As Table~\ref{tb:abla_distill} shows, naive self-distillation can only bring limited improvement. 
These results show that UMT improves overall performance by improving the uni-modal feature learning instead of directly knowledge distillation.

\subsection{Uni-Modal Ensemble in Multi-modal learning}
In this subsection, we demonstrate that Uni-Modal Ensemble is effective on multi-modal datasets where  modalities have strong uni-modal features, outperforming other complex methods. \emph{Even though we don't combine these uni-modal predictions in any special way, but simply average.}

\textbf{In UCF101}, we compare Uni-Modal Ensemble with various multi-modal late-fusion methods. As Table~\ref{tb:ucf-ume} shows, although Gradient Blending~\citep{wang2020makes} and OGM-GE~\citep{peng2022balanced} outperforms baseline methods, they are far worse than Uni-Modal Ensemble.

\textbf{In ModelNet40}, the main comparing methods come from \citet{wu2022characterizing}, which uses a multi-modal DNN with intermediate fusion. It proposes a balanced multi-modal algorithm which balances conditional utilization of each modality by re-balancing the optimization step.
UME surpass their balanced multi-modal algorithm, as Table~\ref{tb:modelnet_ume} shows.

\subsection{Practicality and Reproducibility}

Besides showing better performance, our method does not need an extra split of data to re-weight the losses and re-train the model~\cite{wang2020makes} or tuning too many hyper-parameters~\cite{peng2022balanced}. Compared with naive late-fusion learning, we only need to tune one more hyper-parameter: the weight of distillation loss in UMT. 
As for reproducing our experimental results, the details can be found at Sec~\ref{sec3_3: pushing_force} and Appendix~\ref{app_hyper}. We also provide our code in supplementary material in case of ambiguity.
\section{Conclusion}
This paper analyzes Modality Laziness in multi-modal training and proves that it does hurt the overall performance. 
We propose to choose proper learning method from UME and proposed UMT according to the distribution of uni-modal and paired features and demonstrate its effectiveness.

\section*{Acknowledgments}
This work is supported by the National Key R\&D Program of China (2022ZD0161700).

\nocite{langley00}

\bibliography{example_paper}

\begin{thebibliography}{47}
\providecommand{\natexlab}[1]{#1}
\providecommand{\url}[1]{\texttt{#1}}
\expandafter\ifx\csname urlstyle\endcsname\relax
  \providecommand{\doi}[1]{doi: #1}\else
  \providecommand{\doi}{doi: \begingroup \urlstyle{rm}\Url}\fi

\bibitem[Agrawal et~al.(2018)Agrawal, Batra, Parikh, and
  Kembhavi]{agrawal2018don}
Agrawal, A., Batra, D., Parikh, D., and Kembhavi, A.
\newblock Don't just assume; look and answer: Overcoming priors for visual
  question answering.
\newblock In \emph{Proceedings of the IEEE Conference on Computer Vision and
  Pattern Recognition}, pp.\  4971--4980, 2018.

\bibitem[Allen-Zhu \& Li(2020)Allen-Zhu and Li]{allen2020towards}
Allen-Zhu, Z. and Li, Y.
\newblock Towards understanding ensemble, knowledge distillation and
  self-distillation in deep learning.
\newblock \emph{arXiv preprint arXiv:2012.09816}, 2020.

\bibitem[Amini et~al.(2009)Amini, Usunier, and Goutte]{amini2009learning}
Amini, M.~R., Usunier, N., and Goutte, C.
\newblock Learning from multiple partially observed views-an application to
  multilingual text categorization.
\newblock \emph{Advances in neural information processing systems},
  22:\penalty0 28--36, 2009.

\bibitem[Arora et~al.(2016)Arora, Mianjy, and Marinov]{arora2016stochastic}
Arora, R., Mianjy, P., and Marinov, T.
\newblock Stochastic optimization for multiview representation learning using
  partial least squares.
\newblock In \emph{International Conference on Machine Learning}, pp.\
  1786--1794. PMLR, 2016.

\bibitem[Buciluǎ et~al.(2006)Buciluǎ, Caruana, and
  Niculescu-Mizil]{bucilu2006model}
Buciluǎ, C., Caruana, R., and Niculescu-Mizil, A.
\newblock Model compression.
\newblock In \emph{Proceedings of the 12th ACM SIGKDD international conference
  on Knowledge discovery and data mining}, pp.\  535--541, 2006.

\bibitem[Chen et~al.(2020{\natexlab{a}})Chen, Jain, Schissler, Gari, Al-Halah,
  Ithapu, Robinson, and Grauman]{chen2020soundspaces}
Chen, C., Jain, U., Schissler, C., Gari, S. V.~A., Al-Halah, Z., Ithapu, V.~K.,
  Robinson, P., and Grauman, K.
\newblock Soundspaces: Audio-visual navigation in 3d environments.
\newblock In \emph{Proceedings of the European Conference on Computer Vision
  (ECCV)}, 2020{\natexlab{a}}.

\bibitem[Chen et~al.(2020{\natexlab{b}})Chen, Xie, Vedaldi, and
  Zisserman]{chen2020vggsound}
Chen, H., Xie, W., Vedaldi, A., and Zisserman, A.
\newblock Vggsound: A large-scale audio-visual dataset.
\newblock In \emph{ICASSP 2020-2020 IEEE International Conference on Acoustics,
  Speech and Signal Processing (ICASSP)}, pp.\  721--725. IEEE,
  2020{\natexlab{b}}.

\bibitem[Chen et~al.(2020{\natexlab{c}})Chen, Kornblith, Norouzi, and
  Hinton]{chen2020simple}
Chen, T., Kornblith, S., Norouzi, M., and Hinton, G.
\newblock A simple framework for contrastive learning of visual
  representations.
\newblock In \emph{International conference on machine learning}, pp.\
  1597--1607. PMLR, 2020{\natexlab{c}}.

\bibitem[Cheuk et~al.(2020)Cheuk, Anderson, Agres, and
  Herremans]{cheuk2020nnaudio}
Cheuk, K.~W., Anderson, H., Agres, K., and Herremans, D.
\newblock nnaudio: An on-the-fly gpu audio to spectrogram conversion toolbox
  using 1d convolutional neural networks.
\newblock \emph{IEEE Access}, 8:\penalty0 161981--162003, 2020.

\bibitem[Fayek \& Kumar(2020)Fayek and Kumar]{fayek2020large}
Fayek, H.~M. and Kumar, A.
\newblock Large scale audiovisual learning of sounds with weakly labeled data.
\newblock \emph{arXiv preprint arXiv:2006.01595}, 2020.

\bibitem[Feichtenhofer et~al.(2016)Feichtenhofer, Pinz, and
  Zisserman]{feichtenhofer2016convolutional}
Feichtenhofer, C., Pinz, A., and Zisserman, A.
\newblock Convolutional two-stream network fusion for video action recognition.
\newblock In \emph{Proceedings of the IEEE conference on computer vision and
  pattern recognition}, pp.\  1933--1941, 2016.

\bibitem[Garcia et~al.(2018)Garcia, Morerio, and Murino]{garcia2018modality}
Garcia, N.~C., Morerio, P., and Murino, V.
\newblock Modality distillation with multiple stream networks for action
  recognition.
\newblock In \emph{Proceedings of the European Conference on Computer Vision
  (ECCV)}, pp.\  103--118, 2018.

\bibitem[Gou et~al.(2021)Gou, Yu, Maybank, and Tao]{gou2021knowledge}
Gou, J., Yu, B., Maybank, S.~J., and Tao, D.
\newblock Knowledge distillation: A survey.
\newblock \emph{International Journal of Computer Vision}, 129\penalty0
  (6):\penalty0 1789--1819, 2021.

\bibitem[Gupta et~al.(2016)Gupta, Hoffman, and Malik]{gupta2016cross}
Gupta, S., Hoffman, J., and Malik, J.
\newblock Cross modal distillation for supervision transfer.
\newblock In \emph{Proceedings of the IEEE conference on computer vision and
  pattern recognition}, pp.\  2827--2836, 2016.

\bibitem[Hessel \& Lee(2020)Hessel and Lee]{hessel2020does}
Hessel, J. and Lee, L.
\newblock Does my multimodal model learn cross-modal interactions? it's harder
  to tell than you might think!
\newblock \emph{arXiv preprint arXiv:2010.06572}, 2020.

\bibitem[Hinton et~al.(2015)Hinton, Vinyals, and Dean]{hinton2015distilling}
Hinton, G., Vinyals, O., and Dean, J.
\newblock Distilling the knowledge in a neural network.
\newblock \emph{arXiv preprint arXiv:1503.02531}, 2015.

\bibitem[Hu et~al.(2019)Hu, Yang, Fei, and Wang]{hu2019acnet}
Hu, X., Yang, K., Fei, L., and Wang, K.
\newblock Acnet: Attention based network to exploit complementary features for
  rgbd semantic segmentation.
\newblock In \emph{2019 IEEE International Conference on Image Processing
  (ICIP)}, pp.\  1440--1444. IEEE, 2019.

\bibitem[Huang et~al.(2021)Huang, Du, Xue, Chen, Zhao, and
  Huang]{huang2021makes}
Huang, Y., Du, C., Xue, Z., Chen, X., Zhao, H., and Huang, L.
\newblock What makes multimodal learning better than single (provably).
\newblock \emph{arXiv preprint arXiv:2106.04538}, 2021.

\bibitem[Huang et~al.(2022)Huang, Lin, Zhou, Yang, and
  Huang]{huang2022modality}
Huang, Y., Lin, J., Zhou, C., Yang, H., and Huang, L.
\newblock Modality competition: What makes joint training of multi-modal
  network fail in deep learning?(provably).
\newblock \emph{arXiv preprint arXiv:2203.12221}, 2022.

\bibitem[Kay et~al.(2017)Kay, Carreira, Simonyan, Zhang, Hillier,
  Vijayanarasimhan, Viola, Green, Back, Natsev, et~al.]{kay2017kinetics}
Kay, W., Carreira, J., Simonyan, K., Zhang, B., Hillier, C., Vijayanarasimhan,
  S., Viola, F., Green, T., Back, T., Natsev, P., et~al.
\newblock The kinetics human action video dataset.
\newblock \emph{arXiv preprint arXiv:1705.06950}, 2017.

\bibitem[Langley(2000)]{langley00}
Langley, P.
\newblock Crafting papers on machine learning.
\newblock In Langley, P. (ed.), \emph{Proceedings of the 17th International
  Conference on Machine Learning (ICML 2000)}, pp.\  1207--1216, Stanford, CA,
  2000. Morgan Kaufmann.

\bibitem[Liang et~al.(2021)Liang, Lyu, Fan, Wu, Cheng, Wu, Chen, Wu, Lee, Zhu,
  et~al.]{liang2021multibench}
Liang, P.~P., Lyu, Y., Fan, X., Wu, Z., Cheng, Y., Wu, J., Chen, L., Wu, P.,
  Lee, M.~A., Zhu, Y., et~al.
\newblock Multibench: Multiscale benchmarks for multimodal representation
  learning.
\newblock \emph{arXiv preprint arXiv:2107.07502}, 2021.

\bibitem[Liang et~al.(2022)Liang, Lyu, Chhablani, Jain, Deng, Wang, Morency,
  and Salakhutdinov]{liang2022multiviz}
Liang, P.~P., Lyu, Y., Chhablani, G., Jain, N., Deng, Z., Wang, X., Morency,
  L.-P., and Salakhutdinov, R.
\newblock Multiviz: An analysis benchmark for visualizing and understanding
  multimodal models.
\newblock \emph{arXiv preprint arXiv:2207.00056}, 2022.

\bibitem[Luo et~al.(2018)Luo, Hsieh, Jiang, Niebles, and Fei-Fei]{luo2018graph}
Luo, Z., Hsieh, J.-T., Jiang, L., Niebles, J.~C., and Fei-Fei, L.
\newblock Graph distillation for action detection with privileged modalities.
\newblock In \emph{Proceedings of the European Conference on Computer Vision
  (ECCV)}, pp.\  166--183, 2018.

\bibitem[Nagrani et~al.(2021)Nagrani, Yang, Arnab, Jansen, Schmid, and
  Sun]{nagrani2021attention}
Nagrani, A., Yang, S., Arnab, A., Jansen, A., Schmid, C., and Sun, C.
\newblock Attention bottlenecks for multimodal fusion.
\newblock \emph{Advances in Neural Information Processing Systems}, 34, 2021.

\bibitem[Neverova et~al.(2015)Neverova, Wolf, Taylor, and
  Nebout]{neverova2015moddrop}
Neverova, N., Wolf, C., Taylor, G., and Nebout, F.
\newblock Moddrop: adaptive multi-modal gesture recognition.
\newblock \emph{IEEE Transactions on Pattern Analysis and Machine
  Intelligence}, 38\penalty0 (8):\penalty0 1692--1706, 2015.

\bibitem[Ngiam et~al.(2011)Ngiam, Khosla, Kim, Nam, Lee, and
  Ng]{ngiam2011multimodal}
Ngiam, J., Khosla, A., Kim, M., Nam, J., Lee, H., and Ng, A.~Y.
\newblock Multimodal deep learning.
\newblock In \emph{ICML}, 2011.

\bibitem[Panda et~al.(2021)Panda, Chen, Fan, Sun, Saenko, Oliva, and
  Feris]{panda2021adamml}
Panda, R., Chen, C.-F., Fan, Q., Sun, X., Saenko, K., Oliva, A., and Feris, R.
\newblock Adamml: Adaptive multi-modal learning for efficient video
  recognition.
\newblock \emph{arXiv preprint arXiv:2105.05165}, 2021.

\bibitem[Park et~al.(2017)Park, Hong, and Lee]{park2017rdfnet}
Park, S.-J., Hong, K.-S., and Lee, S.
\newblock Rdfnet: Rgb-d multi-level residual feature fusion for indoor semantic
  segmentation.
\newblock In \emph{Proceedings of the IEEE international conference on computer
  vision}, pp.\  4980--4989, 2017.

\bibitem[Peng et~al.(2022)Peng, Wei, Deng, Wang, and Hu]{peng2022balanced}
Peng, X., Wei, Y., Deng, A., Wang, D., and Hu, D.
\newblock Balanced multimodal learning via on-the-fly gradient modulation.
\newblock \emph{arXiv preprint arXiv:2203.15332}, 2022.

\bibitem[Pezeshki et~al.(2020)Pezeshki, Kaba, Bengio, Courville, Precup, and
  Lajoie]{pezeshki2020gradient}
Pezeshki, M., Kaba, S.-O., Bengio, Y., Courville, A., Precup, D., and Lajoie,
  G.
\newblock Gradient starvation: A learning proclivity in neural networks.
\newblock \emph{arXiv preprint arXiv:2011.09468}, 2020.

\bibitem[Pham et~al.(2019)Pham, Liang, Manzini, Morency, and
  P{\'o}czos]{pham2019found}
Pham, H., Liang, P.~P., Manzini, T., Morency, L.-P., and P{\'o}czos, B.
\newblock Found in translation: Learning robust joint representations by cyclic
  translations between modalities.
\newblock In \emph{Proceedings of the AAAI Conference on Artificial
  Intelligence}, volume~33, pp.\  6892--6899, 2019.

\bibitem[Radford et~al.(2021)Radford, Kim, Hallacy, Ramesh, Goh, Agarwal,
  Sastry, Askell, Mishkin, Clark, et~al.]{radford2021learning}
Radford, A., Kim, J.~W., Hallacy, C., Ramesh, A., Goh, G., Agarwal, S., Sastry,
  G., Askell, A., Mishkin, P., Clark, J., et~al.
\newblock Learning transferable visual models from natural language
  supervision.
\newblock \emph{arXiv preprint arXiv:2103.00020}, 2021.

\bibitem[Romero et~al.(2014)Romero, Ballas, Kahou, Chassang, Gatta, and
  Bengio]{romero2014fitnets}
Romero, A., Ballas, N., Kahou, S.~E., Chassang, A., Gatta, C., and Bengio, Y.
\newblock Fitnets: Hints for thin deep nets.
\newblock \emph{arXiv preprint arXiv:1412.6550}, 2014.

\bibitem[Seichter et~al.(2020)Seichter, K{\"o}hler, Lewandowski, Wengefeld, and
  Gross]{esanet2020}
Seichter, D., K{\"o}hler, M., Lewandowski, B., Wengefeld, T., and Gross, H.-M.
\newblock Efficient rgb-d semantic segmentation for indoor scene analysis.
\newblock \emph{arXiv preprint arXiv:2011.06961}, 2020.

\bibitem[Shah et~al.(2020)Shah, Tamuly, Raghunathan, Jain, and
  Netrapalli]{shah2020pitfalls}
Shah, H., Tamuly, K., Raghunathan, A., Jain, P., and Netrapalli, P.
\newblock The pitfalls of simplicity bias in neural networks.
\newblock \emph{arXiv preprint arXiv:2006.07710}, 2020.

\bibitem[Silberman et~al.(2012)Silberman, Hoiem, Kohli, and
  Fergus]{silberman2012indoor}
Silberman, N., Hoiem, D., Kohli, P., and Fergus, R.
\newblock Indoor segmentation and support inference from rgbd images.
\newblock In \emph{European conference on computer vision}, pp.\  746--760.
  Springer, 2012.

\bibitem[Smith \& Gasser(2005)Smith and Gasser]{smith2005development}
Smith, L. and Gasser, M.
\newblock The development of embodied cognition: Six lessons from babies.
\newblock \emph{Artificial life}, 11\penalty0 (1-2):\penalty0 13--29, 2005.

\bibitem[Soomro et~al.(2012)Soomro, Zamir, and Shah]{soomro2012ucf101}
Soomro, K., Zamir, A.~R., and Shah, M.
\newblock Ucf101: A dataset of 101 human actions classes from videos in the
  wild.
\newblock \emph{arXiv preprint arXiv:1212.0402}, 2012.

\bibitem[Srivastava et~al.(2014)Srivastava, Hinton, Krizhevsky, Sutskever, and
  Salakhutdinov]{srivastava2014dropout}
Srivastava, N., Hinton, G., Krizhevsky, A., Sutskever, I., and Salakhutdinov,
  R.
\newblock Dropout: a simple way to prevent neural networks from overfitting.
\newblock \emph{The journal of machine learning research}, 15\penalty0
  (1):\penalty0 1929--1958, 2014.

\bibitem[Tan \& Bansal(2020)Tan and Bansal]{tan2020vokenization}
Tan, H. and Bansal, M.
\newblock Vokenization: Improving language understanding with contextualized,
  visual-grounded supervision.
\newblock \emph{arXiv preprint arXiv:2010.06775}, 2020.

\bibitem[Tian et~al.(2019)Tian, Krishnan, and Isola]{tian2019contrastive}
Tian, Y., Krishnan, D., and Isola, P.
\newblock Contrastive representation distillation.
\newblock \emph{arXiv preprint arXiv:1910.10699}, 2019.

\bibitem[Wang et~al.(2020{\natexlab{a}})Wang, Zhan, Thompson, and
  Zhou]{wang2020multimodal}
Wang, Q., Zhan, L., Thompson, P., and Zhou, J.
\newblock Multimodal learning with incomplete modalities by knowledge
  distillation.
\newblock In \emph{Proceedings of the 26th ACM SIGKDD International Conference
  on Knowledge Discovery \& Data Mining}, pp.\  1828--1838, 2020{\natexlab{a}}.

\bibitem[Wang et~al.(2020{\natexlab{b}})Wang, Tran, and Feiszli]{wang2020makes}
Wang, W., Tran, D., and Feiszli, M.
\newblock What makes training multi-modal classification networks hard?
\newblock In \emph{Proceedings of the IEEE/CVF Conference on Computer Vision
  and Pattern Recognition}, pp.\  12695--12705, 2020{\natexlab{b}}.

\bibitem[Wu et~al.(2022)Wu, Jastrzebski, Cho, and Geras]{wu2022characterizing}
Wu, N., Jastrzebski, S., Cho, K., and Geras, K.~J.
\newblock Characterizing and overcoming the greedy nature of learning in
  multi-modal deep neural networks.
\newblock In \emph{International Conference on Machine Learning}, pp.\
  24043--24055. PMLR, 2022.

\bibitem[Xiao et~al.(2020)Xiao, Lee, Grauman, Malik, and
  Feichtenhofer]{xiao2020audiovisual}
Xiao, F., Lee, Y.~J., Grauman, K., Malik, J., and Feichtenhofer, C.
\newblock Audiovisual slowfast networks for video recognition.
\newblock \emph{arXiv preprint arXiv:2001.08740}, 2020.

\bibitem[Xu et~al.(2013)Xu, Tao, and Xu]{xu2013survey}
Xu, C., Tao, D., and Xu, C.
\newblock A survey on multi-view learning.
\newblock \emph{arXiv preprint arXiv:1304.5634}, 2013.

\end{thebibliography}
\bibliographystyle{icml2023}

\newpage
\appendix
\onecolumn


\section{Experimental Details and Additional Experiments}
\subsection{Datasets}
\label{appendix:dataset}

Here, we describe the preprocessing of Kinetics-400, VGG-Sound, UCF101 and ModelNet40 in detail.

\textbf{Kinetics-400} dataset~\citep{kay2017kinetics} contains over 240k videos for training and 19k for validation, which we download from cvdfoundation ~\footnote{https://github.com/cvdfoundation/kinetics-dataset}. Kinetics-400 is a commonly used dataset with 400 classes, and we mainly follow the open source preprocessing methods to process that. For RGB modality, we follow the procedure of PySlowFast~\footnote{https://github.com/facebookresearch/SlowFast/}, which resizes the video to the short edge size of 256.
and for audio modality, we follow mmaction2~\footnote{https://github.com/open-mmlab/mmaction2/blob/master/tools/data/build\_audio\_features.py/} to extract specgram features. When performing joint training, we take consecutive 64 frames from a video with fps of 30 and random crop the video to 224*224, and for audio inputs, we take the specgram that can be aligned in time with the clip extracted from the video. When testing, we ensemble the predictions from uniformly sampled clips with RGB and audio from a video and give the final outputs, following PySlowfast.

\textbf{VGG-Sound} dataset \citep{chen2020vggsound}, which contains over 200k video clips for 309 different sound classes, is also used for evaluating our method. 
It is an audio-visual dataset \textit{in the wild} where each object that emits sound is also visible in the corresponding video clip, making it suitable for scene classification tasks. Please note that some clips in the dataset are no longer available on YouTube, and we actually use about 175k videos for training and 15k for testing, but the number of classes remains the same. We design a preprocessing paradigm to improve training efficiency as follows: (1) each video is interpolated to 256$\times$256 and saved as stacked images; (2) each audio is first converted to 16 kHz and 32-bit precision in the floating-point PCM format, then randomly cropped or tiled to a fixed duration of 10s. For video input, 32 frames are uniformly sampled from each clip before feeding to the video encoder. While for the audio input, a 1024-point discrete Fourier transform is performed using nnAudio \citep{cheuk2020nnaudio}, with 64 ms frame length and 32 ms frame-shift. And we only feed the magnitude spectrogram to the audio encoder. 

\textbf{UCF101} dataset \citep{soomro2012ucf101} is an action recognition dataset with 101 action categories, including 7k videos for training and 3k for testing. And we use the rgb and flow provided by~\citet{feichtenhofer2016convolutional}. 
For RGB, we use one image of $(3* 224* 224)$ as the input; while for flow, we use a stack of optical flow images which contained 10 x-channel and 10 y-channel images, So its input shape is $(20 * 224 * 224)$. During training, we perform random crop and random horizontal flip as the data augmentation; while testing, we resize the image to $224$ and do not perform data augmentation operations.

\textbf{ModalNet40} is a 3D object classification dataset with 9,483 training samples and 2,468 test samples. We base on the front view and the rear view of the 3D object to classify that, following~\citet{wu2022characterizing}.

\subsection{Training Hyperparameters}
\label{app_hyper}
In VGG-Sound, UCF101 and ModelNet40, we use ResNet as our backbone, all with 18 layers. As for Kinetics-400, we use 50 or 101 layers' ResNet to encode the inputs. Noting that 3D CNN is used for visual data of VGG-Sound and Kinetics-400. 

In UMT, we use cross-entropy loss as the task loss and MSELoss as the distillations loss. The weights of task loss and distillation loss are 1 and 50, respectively. In UME, we directly average the uni-modal models' predictions and no additional training is added.

We show the hyperparameters of our experiments in UCF101 and VGG-Sound in Table~\ref{tb:hyper-details}.

As for Kinetics-400's RGB modality, we totally follow the hyperparameters and settings of PySlowFast~\footnote{https://github.com/facebookresearch/SlowFast/configs/Kinetics/SLOWFAST\_8x8\_R50.yaml}. 
As for audio modality, we modify the hyperparameters~\footnote{open\-mmlab/mmaction2/blob/master/configs/recognition\_audio/resnet/tsn\_r18\_64x1x1\_100e\_kinetics400\_audio\_feature.py} to be as consistent as possible with the RGB training for further joint training. Specifically, we use the same learning rate and batch size as RGB training used. 

As for ModelNet40, we totally follow the experimental settings of \citet{wu2022characterizing}~\footnote{https://github.com/nyukat/greedy\_multimodal\_learning}.

\begin{table}[t]
\centering
\caption{The Hyperparameters used in our experiments for VGG-Sound and UCF101.}
\vspace{5pt}
\begin{tabular}{lll}
\toprule
Hyperparameter & Value (VGG-Sound) & Value (UCF101) \\

\specialrule{0.05em}{2pt}{2pt} 
Encoder & ResNet3D (Video), 2D (Audio) & ResNet2D(Both Modalities)\\
Linear Head & (1024, 309) & (1024, 101)\\
MLP Head & (1024, 1024)& (1024, 1024) \\
    & ReLU & ReLU\\
    &(1024, 309) & (1024, 101)\\
 Attension Head &   \multicolumn{2}{c}{ Attension Layer (without new parameters) + a linear layer} \\
Training Epoches & 20 & 20 \\
LR & 1e-3 & 1e-2 \\
Batch Size & 24 & 64 \\
Optimizer & Adam & SGD\\
Scheduler & StepLR (step=10, gamma=0.1) & ReduceLROnPlateau (patience=1)\\
Loss Fusion & \multicolumn{2}{c}{ Cross Entropy for task, MSE for distillation}\\
\bottomrule
\end{tabular}
\label{tb:hyper-details}
\end{table}

\subsection{Can existing optimizers solve Modality Laziness?}
\label{app_opt}

\begin{table*}[ht]
	\centering
	\renewcommand{\arraystretch}{1.1}
	\caption{Top-1 test accuracy (in \%) of linear classifiers trained on frozen encoders from multi-modal late-fusion training under different optimizers and uni-modal training on VGG-Sound. }
	\vspace{5pt}
	\begin{tabular}{c| c |c c }
		\toprule
		Optimizer & Multi-modal Performance & Audio Encoder & RGB Encoder \\
		
		\midrule
		
		SGD & 47.13 & 40.02  & 15.53\\
		RMSprop & 47.90 & 42.77  & 13.64 \\
		Adagrad & 42.19 & 35.68  & 19.65\\
		Adadelta & 23.18 & 17.70  & 17.37\\
		Adamw & 49.39 & 42.41  & 15.11\\
		Adam & 49.47 & 43.44  & 15.56\\
		
		\midrule
		Uni-Training & / & \textbf{45.15} & \textbf{23.17}  \\
		
		\bottomrule
	\end{tabular}
	\label{tb:different_optimizers}
\end{table*}

While the results in Table~\ref{tb:modality_laziness} show that different multi-modal methods suffer from learning insufficient uni-modal features, \emph{how about changing the optimizer}? To answer this question, we try different optimizers for multi-modal late-fusion training~(with a linear multi-modal head), including SGD, RMSprop, Adagrad, Adadelta, Adamw and Adam. As Table~\ref{tb:different_optimizers} shows, Modality Laziness exists no matter which optimizer is used.

\subsection{Details on Uni-Modal Teacher~(UMT)}
\label{app_umt_details}
In this subsection, we describe how Uni-Molda Teacher (UMT) applies on multi-modal late-fusion tasks. The overall architecture can be found in Figure~\ref{fig:naive_fusion}.

\textbf{UMT in late-fusion classification.} In multi-modal late-fusion architecture, modalities are first encoded by the corresponding encoders and then mapped to the output space by a multi-modal fusion head (Figure~\ref{fig:naive_fusion} left). 
Uni-Modal Teacher distills the pre-trained uni-modal features to the corresponding parts in multi-modal networks in multi-modal training (Figure~\ref{fig:naive_fusion} right). Uni-modal distillation happens before fusion, so it's suitable for late-fusion multi-modal architecture. 
The pre-trained uni-modal features are generated by inputting the data to the pre-trained uni-modal models.

\begin{figure}[t]
    \centering
    \includegraphics[width=0.8\textwidth]{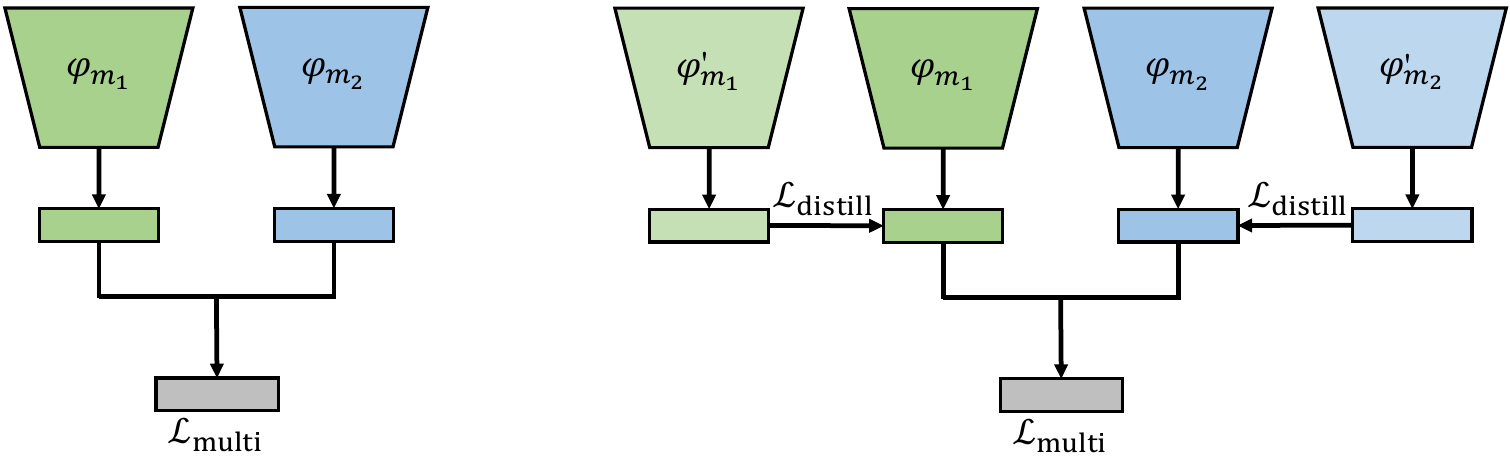}
    \caption{Model architecture of naive late fusion (left) and Uni-Modal Teacher (UMT) (right). $\varphi^{'}_{m_i}$ is the encoder which is supervised pre-trained on uni-modal data. $\varphi_{m_i}$ is a random initialed encoder without pre-training. $\mathcal{L}_{multi}$ is the loss between multi-modal predictions and labels. $\mathcal{L}_{distill}$ is the uni-modal distillation loss.}
    \label{fig:naive_fusion}
\end{figure}

\textbf{UMT's weights.} For VGG-Sound and Kinetics, we use $50$ (both audio feature distillation and RGB feature distillation) as the distillation loss's weight. 
We test different distillation weights on VGG-Sound and Kinetics-400. As shown in Table~\ref{tb:distill_weights}, on both datasets, UMT performs well with an distillation weight of 50.

\begin{table}[h]
    \centering
    \caption{Different distillation weights of UMT on VGG-Sound and Kinetics-400}
    \vspace{5pt}
    \begin{tabular}{ccccccc}
    \hline
    Dataset     & 0&  1 & 10& 20 & 50 & 100 \\
    \hline
    VGG-Sound & 49.46& 49.51 & 51.31 & 51.51
 & \textbf{53.46} & 53.11\\
    Kinetics-400 & 74.25 & 74.99 & 75.57& 76.11& \textbf{76.77}& 76.55\\
    \hline
    \end{tabular}
    \label{tb:distill_weights}
\end{table}

\subsection{Dropout in Multi-modal Training.}
\label{app_drop}

\begin{table}[]
    \centering
    \caption{Dropout in multi-modal training on VGG-Sound.}
	\vspace{5pt}
	\begin{tabular}{c| c }
		\toprule
		Method & Performance\\
		\midrule
	    Baseline & 49.46\\
	    Dropout& 49.83  \\
	    Modal-Drop& 51.37\\
	    \midrule
	    UMT&  \textbf{53.46}\\
		\bottomrule
	\end{tabular}
	\label{tab:dropout_vgg}
\end{table}

Here we consider the common regularizer, dropout~\citep{srivastava2014dropout}, and a variant of it, namely modality-wise dropout, which randomly drops (with probability $1/3$) the feature from one modality in each iteration. Modality dropout is akin to the ModDrop in \cite{neverova2015moddrop}. 
As Table~\ref{tab:dropout_vgg} shows, modality-wise dropout is significantly better than dropout, which implies that modality-wise laziness is serious and modality-wise dropout is also effective.

\subsection{Finetuning the uni-modal pre-trained encoders}
\label{app:finetune}

\begin{table}[ht]
   \centering
   \caption{The top-1 test accuracy of finetuning the uni-modal pre-trained encoders and linear evaluation on finetuned encoders on VGG-Sound. }
   \vspace{5pt}
    \begin{tabular}[htb]{cccc}
    \toprule
        \multirow{2}*{\textbf{Encoder LR}} & \multirow{2}*{\textbf{Top-1 Acc}} & \multicolumn{2}{c}{\textbf{Encoder Eval}}  \\
         &  & Audio & RGB \\
         \midrule
         1e-3& 50.98 & 43.98& 21.86 \\
         1e-4& 49.37& 44.71& 21.97\\
         1e-5& 50.45& 45.28& 23.13 \\
         1e-6& 50.86&45.29 & 23.27\\
         \midrule
         0&50.95& 45.15& 23.17\\
         \bottomrule
         
    \end{tabular}
    \label{tb:finetuning_encoder_app}
\end{table}

In this subsection, we use the uni-modal pre-trained encoders' parameters as the initialized weights in multi-modal training and randomly initialize a multi-modal linear classifier on the encoders. 
We set the classifier's learning rate as $1e-3$ and try different learning rates on the encoders.

As Table~\ref{tb:finetuning_encoder_app} shows, using the uni-modal  supervised pre-trained encoder's weights in multi-modal training and then fine-tuning the whole multi-modal model can bring some improvement compared to naive fusion~($49.46$) but is worse than UMT, which gets $53.46$ accuracy. 
When the learning rate of encoders is large, the encoders forget some abilities to extract uni-modal features.

In this paper, we focus on supervised learning and unsupervised pre-training\cite{ngiam2011multimodal} may be also helpful but beyond our scope.


\subsection{The role of cross-modal interactions on different datasets}
\label{app_class}

In this subsection, we conduct various experiments to further investigate the effect of cross-modal interactions and explore the benefits and harms that cross-modal interaction brings in different multi-modal tasks/datasets. We find that in different datasets, cross-modal interaction has a different effect on the performance.

\subsubsection{Averaging the uni-modal predictions \emph{vs} The linear classifier trained on uni-modal pre-trained encoders.}
\label{app_uniclf_from_mmclf}

In Sec~4.2, we train a multi-modal linear classifier on frozen uni-modal pre-trained encoders and compare this classifier with directly averaging uni-modal models' predictions.
As Table~\ref{tb:avg_linearclf} shows, this classifier does not consistently outperform simply averaging the uni-modal predictions on all datasets.
It shows better performance on VGG-Sound and Kinetics-400, but worse performance on UCF101 and ModelNet40.

To further explain this phenomenon, we check and disassemble this new trained multi-modal classifier on UCF101.
In the late-fusion multi-modal training, the features of different modalities are concatenated first and then the multi-modal classifier receives them and output predictions. 
Different modalities in the classifier do not share the parameters.  
So we split the new trained multi-modal linear classifier into uni-modal classifiers.
We use the uni-modal pre-trained encoders to extract features and then the uni-modal classifiers receive the corresponding features and output predictions.
Noting that OGM-GE~\citep{peng2022balanced} uses similar technique to check how well different modalities are trained. 
As Table~\ref{tb:uni_from_mm} shows, the uni-modal classifiers from new trained multi-modal classifiers are significantly worse than uni-modal models, implying that the multi-modal classifier trained on uni-modal pre-trained encoders suffers from serious Modality Laziness on UCF101, although it is just a simple linear layer, resulting in worse performance than directly averaging the uni-modal predictions.

\subsubsection{Class-level Evaluation on Different Multi-modal Datasets}
In this subsection, we compare naive late-fusion learning with averaging predictions of uni-modal models in \emph{class} level. It's obvious that there are more cross-modal interactions in naive fusion. 

Although naive fusion suffers from learning insufficient uni-modal features, we find in some classes in Kinetics and VGG-Sound, the accuracy of naive fusion model outperforms averaging the uni-modal models' predictions, and even outperforms the sum of the accuracy of the two uni-modal models in VGG-Sound and Kinetics-400, as shown in Table~\ref{tb:class} and \ref{tb:kinetics_class}.

However, 
We cannot find any class that naive fusion can exceed the sum of the accuracy of the uni-RGB model and uni-flow model in UCF101. 
We select classes in UCF101 by sorting the differences of accuracy between naive fusion and the best uni-modal model in class level and the top ten with the largest difference are selected. 
In these classes where naive fusion has advantages, averaging the predictions can outperform naive fusion in some classes~(ID:29, 67, 71), and this phenomenon is not found in VGG-Sound and Kinetics.
And as Tabla~\ref{tb:uni_from_mm} and Table~\ref{tb:ucf_class} show, both RGB and optical flow in UCF101 can get strong performance individually.
All the evidence shows that in UCF101, the uni-modal features are totally dominate and any joint training can lead to serious Modality Laziness. 

\begin{table}[t]
    \centering
    \caption{Top-1 test accuracy of different models on some classes of Kinetics. The accuracy of naive fusion model outperforms averaging the uni-modal models' predictions, and even outperforms the sum of the accuracy of the uni-audio model and uni-video model.}
    \vspace{5pt}
    \begin{tabular}{l|cccccccccc|c}
    \toprule
      Class ID   &53  &90 &184 &2 &368 & 158 & 113 & 263 & 287 & 4 &mean accuracy  \\
    \midrule   
    Uni-Audio&      0  & 0 &0 & 4   & 0  & 0  & 0 & 0 & 4 & 2& 1\\
    Uni-RGB& 42  & 50 &22& 28 & 39& 43 & 29& 82& 76 & 50& 46.1\\
    \midrule
    Avg Pred &42 &50 &22 &28 &39 &43 &29 &82 &78 &50 & 46.3\\
    Naive Fusion&   \textbf{56} & \textbf{62} &\textbf{32}& \textbf{40} & \textbf{45}& \textbf{49} & \textbf{35}& \textbf{84}& \textbf{86} & \textbf{58}& \textbf{54.7}\\
    \bottomrule
    \end{tabular}
    \label{tb:kinetics_class}
\end{table}

\begin{table}[t]
    \centering
    \caption{Top-1 test accuracy of different models on selected classes of UCF101. We select the top-10 classes according to the gap of accuracy between the multi-modal and uni-modal models. As we can see, uni-modal model's performance is high, meaning paired features in UCF101 are rare.}
    \vspace{5pt}
    \begin{tabular}{l|cccccccccc|c}
    \toprule
      Class ID   & 6  & 10 & 12 & 22 & 29 & 31 & 48 &57 & 67 & 71   & mean accuracy \\
    \midrule   
    Uni-RGB&      74  & 84 & 76 & 61 & 67 & 32 & 78 &21 & 75 & 57   & 62.5  \\ 
    Uni-Flow&     60  & 82 & 58 & 47 &61 & 41 & 64 &24 & 78 & 63   & 57.8 \\
    \midrule
    Avg Pred & 70 & 95 & 79& 64 & \textbf{86} & 46 & 89 & 36 & \textbf{98} & \textbf{83}   & 74.6 \\
    Naive Fusion& \textbf{86}  & \textbf{95} &\textbf{87} &\textbf{72}  & 83 & \textbf{59} &\textbf{92} & \textbf{42} & 88 & 73  & \textbf{77.7} \\
    \bottomrule
    \end{tabular}
    \label{tb:ucf_class}
\end{table}

\paragraph{The mapping betwenn class ID and class name in different datasets}

The correspondence between id and name of the selected class in VGG-Sound is:
164: People Sniggering,
303: Wood Thrush Calling,
33: Cat Meowing,
255: Sea Waves,
91: Footsteps On Snow,
4: Alligators Crocodiles Hissing,
152: People Gargling,
127: Mynah Bird Singing,
68: Door Slamming,
155: People Humming.

For Kinetics-400, we sort the classes alphabetically from smallest to largest according to the class name, and then we can get the mapping between class name and id.

In UCF101, the mapping can be found in classInd.txt, a given file of UCF101.

\subsection{Cross-modal interaction in Uni-Modal Teacher~(UMT)}

\begin{wraptable}{r}{6cm}
    \centering
	\renewcommand{\arraystretch}{1.1}
	\caption{Comparison of UMT with combining uni-modal models trained by distillation on VGG-Sound. }
	\vspace{5pt}
	\begin{tabular}{c| c  c| c }
		\toprule
		Method & RGB & Audio & R+A \\
		\midrule
	    Linear Clf & \textbf{25.99} & \textbf{46.00} & 52.98\\
	    UMT & 24.43 &45.89 & \textbf{53.46}\\
		\bottomrule
	\end{tabular}
	\label{tb:umt_cross}
\end{wraptable}
In order to verify whether the multimodal loss in UMT makes sense,
we train uni-modal models by knowledge distillation to get better performance than encoders trained by UMT and then combine them by introducing a new multi-modal classifier on these encoders. 
As Table~\ref{tb:umt_cross} shows, UMT works better in multi-modal performance, although the encoders of UMT in uni-modal evaluation are worse, showing that UMT indeed benefits from cross-modal interaction.

\subsection{Exploring UMT for Multi-modal Segmentation}
\textbf{NYU Depth V2} dataset \citep{silberman2012indoor} contains 1449 indoor RGB-Depth data totally and we use 40-class label setting. The number of training set and testing set is 795 and 654 respectively. All perprocessing operations are following~\citep{esanet2020}.

In contrast to the late fusion classification task, the RGB-Depth semantic segmentation belongs to middle fusion.
The main encoder receives RGB inputs, and the depth inputs are fed into the depth encoder. At each intermediate layer, the main encoder fuses its own intermediate outputs and the depth features obtained from the depth encoder, which makes it a mid-fusion task~\citep{esanet2020}. 
Since features generated by each layer matter, we distill multi-scale depth feature maps using the MSE loss. For feature maps from the RGB encoder, however, since they are generated by fusing RGB and depth modalities, we cannot distill RGB feature maps directly like depth feature maps. To mitigate this effect, we curate predictors, namely 2 layers CNNs, aiming to facilitate the fused feature maps to predict the RGB feature maps trained by the RGB modality before distillation. The full schematic diagram is presented in Figure~\ref{fig:mid}. As shown in Table~\ref{tb:rgbd-result}, UMT can also improve multi-modal segmentation whether the encoder is pre-trained on ImageNet or not.

\begin{figure}[t]
    \centering
    \includegraphics[width=0.8\textwidth]{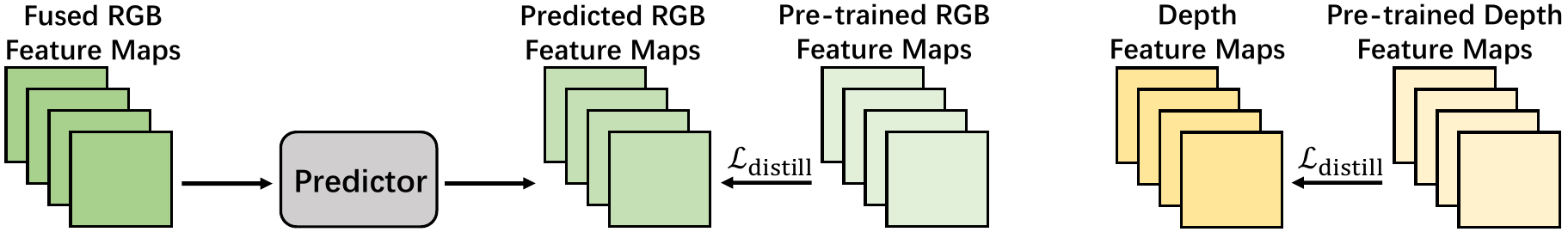}
    \caption{Distillation details of UMT for RGB (left) and depth (right) modalities in multi-modal semantic segmentation (based on ESANet).}
    \label{fig:mid}
\end{figure}

\begin{table}[t]
	\centering
	\renewcommand{\arraystretch}{1.1}
	\caption{Model performance comparison under UMT and ESANet on NYU-DepthV2 RGB-Depth semantic segmentation task.}
	\vspace{5pt}
	\begin{tabular}{c c c }
		\toprule
		\multirow{2}*{\textbf{Initialization}} & \multicolumn{2}{c}{\textbf{Training Setting}}   \\ \cline{2-3} & ESANet & UMT\\
		\midrule
		From Scratch & 38.59 & \textbf{40.45 (+1.86)}  \\
		ImageNet Pre-train  & 48.48  & \textbf{49.39 (+0.91)} \\
		\bottomrule
	\end{tabular}
	\label{tb:rgbd-result}
\end{table}

\subsection{Explanations on Paired Features}
\label{app_paired}
We revisit the definitions of uni-modal features and paired features: \textbf{uni-modal features}, which can be learned by uni-modal training; \textbf{paired features}, which can only be learned by cross-modal interaction in joint training.
Different datasets contain different proportions of these features.

In this subsection, we use synthetic datasets to explain the uni-modal features and paired features in multi-modal tasks. 

\paragraph{Understanding different types of features in multi-modal tasks by synthetic datasets.}
Three different multi-modal datasets are generated to help us understand the uni-modal features and paired features in multi-modal tasks. 
The process of data generation mainly refers to \cite{hessel2020does}.

\textbf{Firstly, we generate a dataset where each modality can extract the features to give correct predictions.} We name this dataset as \textbf{Dataset~$\alpha$}.
The data generation process is as follows:
\begin{enumerate}
    \item Sample random projection $P_1 \in \mathbb{R}^{d_1\times d}$ and $P_2 \in \mathbb{R}^{d_2\times d}$ from $U(-0.5,0.5)$. 
    \item Sample $z\in \mathbb{R}^d\sim \mathcal{N}(0,1)$. Normalize $z$ to unit length
    \item Sample $x\in \mathbb{R}^d\sim \mathcal{N}(0,1)$. Normalize $x$ to unit length
    \item if $|x\cdot z|\leq 0.1$, return to the Step 3.
    \item If $x\cdot z > 0.1$, then $y=1$; else $y=0$.
    \item Get the data point $(P_1x, P_2x, y)$.
    \item If the amount of data generated is less than $N$, return to the Step 3; else break
\end{enumerate}
The $P_1x, P_2x$ represents two modalities of the multi-modal dataset
and we set $d_1, d_2, d, N$ as $200, 100, 50, 5000$, respectively. And we randomly split 80\% of the generated data as train set, and the rest serves as a test set. 
In this dataset, uni-modal models can extract useful features to give correct predictions and multi-modal joint training is not necessary, as Table~\ref{tab:syn} shows~(\textbf{Dataset~$\alpha$}). 
\emph{We name the features, that uni-modal models can learns to give correct predictions in the given task, as uni-modal features.}

\begin{table}[t]
    \centering
    \caption{Test Accuracy of uni-modal models and multi-modal model on different synthetic datasets.
    \textbf{Synthetic Dataset~$\alpha$} mainly contains uni-modal features which can be learned in uni-modal training; 
    \textbf{Synthetic Dataset~$\beta$} mainly contains paired features which need joint training to learn; 
    \textbf{Synthetic Dataset~$\gamma$} contains uni-modal features and paired features.
    }
    \begin{tabular}{c|c| c| c}
    \toprule
        Dataset & \textbf{Synthetic Dataset~$\alpha$} & \textbf{Synthetic Datset~$\beta$} & \textbf{Synthetic Dataset~$\gamma$} \\
    \midrule
        Uni-modal~(Modality 1) &100 & 51.4 & 70.9\\
        Uni-modal~(Modality 2) &100 & 51.8 & 70.1\\
    \midrule
        Multi-modal &100 & 92& 94.4 \\
    \bottomrule
    \end{tabular}
    \label{tab:syn}
\end{table}

\begin{table}[t]
    \centering
    \caption{The confusion matrix of uni-modal model in \textbf{Dataset $\gamma$}. In the data labeled 0, each modality contains features that can give a correct prediction, while in the data labeled 1 or 2, we need both modalities together to make the right predictions.}
    
    \begin{tabular}{cc|ccc}
    \multicolumn{2}{c}{}
             &      \multicolumn{2}{c}{Predicted} \\
        &       &   0 &   1   & 2          \\ 
    \cline{2-5}
    \multirow{3}{*}{\rotatebox[origin=c]{90}{Actual}}
     & 0    & 100\%   & 0   & 0                 \\
     &  1  & 0    & 57\%    & 43\%               \\ 
     &  2 & 0    & 45.4\%    & 54.6\%               \\ 

    \cline{2-5}
    \end{tabular}
    \label{tab:confusion_matrix}
\end{table}

\textbf{Secondly, we generate another dataset where the model must rely on both the two modalities to make correct predictions.} We name this dataset as \textbf{Dataset~$\beta$}.
The data generation process is as follows:
\begin{enumerate}
    \item Sample random projection $P_1 \in \mathbb{R}^{d_1\times d}$ and $P_2 \in \mathbb{R}^{d_2\times d}$ from $U(-0.5,0.5)$. 
    \item Sample $x_1,x_2\in \mathbb{R}^d \sim \mathcal{N}(0,1)$. Normalize $x_1, x_2$ to unit length
    \item if $|x_1\cdot x_2|\leq 0.25$, return to the Step 2.
    \item If $x_1\cdot x_2 > 0.25$, then $y=1$; else $y=0$.
    \item Get the data point $(P_1x_1, P_2x_2, y)$.
    \item If the amount of data generated is less than $N$, return to the Step 2; else break
\end{enumerate}
This multi-modal dataset is different from the first dataset, because the labels in this dataset are highly dependent on the relationship between the two modalities.
As we can see in Table~\ref{tab:syn}~(\textbf{Dataset~$\beta$}), the uni-modal models can only gives about 50 percent accuracy, while the multi-modal models can give about 90 percent accuracy.
In binary classification tasks, 50 percent accuracy is no different from guessing.
\emph{In this dataset, because labels are heavily relied on the relationships of the two modalities, we must train both modalities simultaneously to extract the joint representations to learn the relationship of the two modalities, which are beyond uni-modal features. 
In order to better carry out theoretical analysis, we abstract these representations into paired features, which can only be learned from multi-modal joint training in multi-modal tasks.}

\textbf{Finally, we generate a dataset that contains both uni-modal features and paired features.} We name this dataset as \textbf{Dataset~$\gamma$}.
The data generation process is as follows:
\begin{enumerate}
    \item Sample random projection $P_1 \in \mathbb{R}^{d_1\times d}$ and $P_2 \in \mathbb{R}^{d_2\times d}$ from $U(-0.5,0.5)$. 
    \item Sample $z\in \mathbb{R}^d\sim \mathcal{N}(0,1)$. Normalize $z$ to unit length
    \item Sample $x\in \mathbb{R}^d\sim \mathcal{N}(0,1)$. Normalize $x$ to unit length
    \item if $x\cdot z \geq 0.1$, get the data point $(P_1x, P_2x, y=0)$.  else, return to the step 3 until collecting $2500$ data points.
    \item Sample $x_1,x_2\in \mathbb{R}^d \sim \mathcal{N}(0,1)$. Normalize $x_1, x_2$ to unit length. 
    If $x_1\cdot z > -0.1$ or $x_2\cdot z > -0.1$, resample.
    \item if $|x_1\cdot x_2|\leq 0.25$, return to the Step 5.
    \item If $x_1\cdot x_2 > 0.25$, then $y=2$; else $y=1$.
    \item Get the data point $(P_1x_1, P_2x_2, y)$.
    \item If the total amount of data generated is less than $7500$, return to the Step 5; else break
\end{enumerate}
In the data labeled 0, each modality contains features that can give a correct prediction, while in the data labeled 1 or 2, we need both modalities together to make the right predictions.
To further understand how uni-modal models give predictions in dataset that containing both uni-modal and paired features, we give the confusion matrix of the uni-modal model. As Table~\ref{tab:confusion_matrix} shows, the uni-modal model can give correct predictions for data labeled 0, while for data labeled 1 or 2, it fails and gives a random predictions. Because for data labeled 1 or 2, we need to learn the relationship of the two modalities to give the correct predictions.

In this subsection, we mainly discuss the synthetic multi-modal datasets and in Appendix~\ref{app_class}, we conduct various experiments on real-world multi-modal datasets to help us understand the uni-modal features, paired features and cross-modal interaction in multi-modal training better.

\paragraph{Training settings on synthetics datasets.}
We use a two layer MLP with ReLU as activation function. As for hidden layer, we use 200 dimensions for multi-modal training and 100 dimensions for uni-modal training. We use SGD as the optimizer and the learning rate is 0.2. In each iteration, we use the whole training set to compute the gradients. And we provide the code in supplement materials.

\subsection{Uni-Modal Performance in Kinetics-400}
\label{app_uni_model-kinetics}

Kinetics-400 is a dynamic dataset, because videos may be removed from YouTube. In this subsection, we report the uni-modal performance of ours and \cite{xiao2020audiovisual}'s on Kinetics-400. 
As Table~\ref{tab:kin_uni_modal} shows, we cannot reproduce their uni-modal performance and ours are lower than theirs. But we demonstrate that UMT outperforms AVSlowFast in Sec~4.3.1, which shows UMT's effectiveness.

\begin{table}[ht]
    \centering
    \caption{Uni-Modal Performance of ours and \citet{xiao2020audiovisual}'s on Kinetics-400}
    \vspace{5pt}
    \begin{tabular}{c|cc}
        \toprule
         &  ours & \citet{xiao2020audiovisual}\\
        \midrule
        Uni-Audio & 23.5 & \textbf{24.8}\\
        Uni-RGB~(SlowFast-50) & 74.9 &\textbf{75.6}\\
        Uni-RGB~(SlowFast-101) & 77.2 & \textbf{77.9}\\
        \bottomrule
    \end{tabular}
    \label{tab:kin_uni_modal}
\end{table}
\section{Proof}

\subsection{Proof of Theorem~\ref{thm:combineoverfitting}}
\label{append:proof}

\Combineoverfitting*

We prove the theorem, which shows that naive joint training indeed suffers from overfitting issues, meaning that it learns less features compared to \Vote.

\begin{proof}
We first introduce some additional notations used in the proof.
We define the features trained in $\modalOne$-uni-modal training as $f_1(\modalOne), \dots, f_{\bmone}(\modalOne)$, define the features trained in $\modalOne$-uni-modal training as $g_1(\modalTwo), \dots, g_{\bmtwo}(\modalTwo)$.
Therefore, there are in total $\bmone + \bmtwo$ features learned in \Vote, namely, $f_1(\modalOne), \dots, f_{\bmone}(\modalOne), g_1(\modalTwo), \dots, g_{\bmtwo}(\modalTwo)$.
Besides, We define the features trained in \Combine\ as $f_1(\modalOne), \dots, f_{\kmone}(\modalOne)$, $g_1(\modalTwo), \dots, g_{\kmtwo}(\modalTwo)$, $h_1(\modalOne, \modalTwo), \dots, h_{\kpa}(\modalOne, \modalTwo)$.
When the context is clear, we omit the dependency of $\modalOne, \modalTwo$ and denote them as $f_i, g_i, h_i$ for simplicity. 
When the context is clear, we abuse the notation $r$ to represent arbitrary $f$, $g$ or $h$.
The corresponding predicting probability of feature $r_i$ is denoted as $p(r_i)$.
To summary, there are $\bmone + \bmtwo$ features in \Vote, $\kmone + \kmtwo + \kpa$ features in \Combine.

We first prove statement (a.), which claims that the number of features learned in \Combine\ are provably less than any of the number of features learned in uni-modal training.
The proof depends on the following Lemma~\ref{lem1}.

\begin{lemma}
\label{lem1}
Assume there exists $T$ features $r_i, i=1,\dots, T$.
If we replace one of the $T$ features (without loss of generality, $r_T$) with a more powerful feature $r^\prime$, where $p(r^\prime) > p(r_T)$, then the predicting probability for each data point increases (where the probability is taken over the randomness of the training data).
\end{lemma}

We next provide the proof of statements (a.): based on Lemma~\ref{lem1}.
We shall prove $\kmone + \kmtwo+ \kpa < \bmone$ without loss of generality.
Start from the features $f_1(\modalOne), \dots, f_{\kmone}(\modalOne)$ which are common features in both \Combine\ and Uni-modal training.
Next step, we add feature $f_{\kmone+1}$ in \Uni and $g_1$ in \Combine.
Obviously, $p(g_1) > p(f_{\kmone + 1})$ due to the training priority (or \Combine\ should learn $f_{\kmone + 1}$ instead of $g_1$).
Therefore, the predicting probability of \Combine\ is larger than \Uni.

Repeating the procedure by comparing $g_i$ with $f_{\kmone + i}$ and comparing $h_j$ with $f_{\kmone+\kmtwo+j}$, the predicting probability of \Combine\ is always larger than \Uni.
Note that $\bmone$ should be always larger than $\kmone+\kmtwo$, or the predicting probability of \Uni\ would be smaller than \Combine.
At the end of the comparison, the predicting probability of \Combine\ is still larger than \Uni.
This requires that \Uni\ should learn more features, which can be regarded as \Uni learns a features while \Combine\ learns an empty feature.
In conclusion, \Uni\ learns more features compared to \Combine, leading to $\bmone > \kmone+\kmtwo+\kpa$.

We next prove the statement (b.).
The proof of (b.) is based on (a.).
We next only consider modality $\modalOne$, the proof for modality~$\modalTwo$ is similar.
Note that the since the number of features learned in \Combine\ is less than $\bmone$, the number of features learned in $\modalOne$ must be less than $\bmone$ (Note that those features can be either \pairFeature\ or \singFeature, namely, $f_1, \dots, f_{\kmone}$ and $h_1, \dots, h_{\kpa}$). 
Therefore, \Combine\ learns less features compared to \Uni\ in modality~$\modalOne$.
On the other hand, the predicting probability of features learned in \Combine\ ($f_1, \dots, f_{\kmone}$ and $h_1, \dots, h_{\kpa}$, considering only modality $\modalOne$ for the \pairFeature) is less than that learned in \Uni\ ($f_1, \dots, f_{\bmone}$), because otherwise, \Uni\ will learn the features in $h$ instead of $f$.
In conclusion, when considering only modality $\modalOne$, the number of features learned in \Combine\ is less and its corresponding predicting probability is small.
Therefore, each modality in \Combine\ performs worse than \Uni.

We finally prove the statement (c.).
Recall that the loss is $-\sum_i u(r_i)$ where $u(r_i) = \bI( y r_i>0) - \bI(y r_i<0)$.
Note that $\bE (u(r_i)) = \frac{1}{2} p(r_i)$ and $|u(r_i)| \leq 1$.
We derive that:
\begin{equation*}
\label{eqn:losscompare}
    \begin{split}
        &\bP\left( -\sum_{i \in [\kmone]} u(f_i) - \sum_{i \in [\kmtwo]} u(g_i) - \sum_{i \in [\kpa]} u(h_i)  \leq -\sum_{i \in [\bmone]} u(f_i) - \sum_{i \in [\bmtwo]} u(g_i) \right) \\
        =& \bP\left( \sum_{\kmone < i \leq \bmone} u(f_i) + \sum_{\kmtwo < i \leq \bmtwo} u(g_i) - \sum_{i \in [\kpa]} u(h_i)  \leq 0 \right) \\
        =& \bP\left( \sum_{\kmone < i \leq \bmone} u(f_i) + \sum_{\kmtwo < i \leq \bmtwo} u(g_i) - \sum_{i \in [\kpa]} u(h_i) + \frac{1}{2} E \leq  \frac{1}{2} E \right), \\
    \end{split}
\end{equation*}
where $E  = - \bE (\sum_{\kmone < i \leq \bmone} u(f_i) + \sum_{\kmtwo < i \leq \bmtwo} u(g_i) - \sum_{i \in [\kpa]} u(h_i)) = \sum_{i \in [\kpa]} p(h_i) - \sum_{\kmone < i \leq \bmone} p(f_i) - \sum_{\kmtwo < i \leq \bmtwo} p(g_i)$.
Due to the training priority and the conclusion in (a.), $$\sum_{i \in [\bmone + 1, \bmone + \bmtwo]} p_{[i]} \leq \sum_{\kmone < i \leq \bmone} p(f_i) + \sum_{\kmtwo < i \leq \bmtwo} p(g_i).$$
Therefore, $E \leq \sum_{i \in [\kpa]} p(h_i) - \sum_{i \in [\bmone + 1, \bmone + \bmtwo]} p_{[i]} \leq \sqrt{8 (\kpa + \bmone - \kmone + \bmtwo - \kmtwo) \log(1/\delta)}$.
We next apply Hoeffding inequality on Equation~\ref{eqn:losscompare} and derive that
\begin{equation*}
    \begin{split}
          &\bP\left( -\sum_{i \in [\kmone]} u(f_i) - \sum_{i \in [\kmtwo]} u(g_i) - \sum_{i \in [\kpa]} u(h_i)  < -\sum_{i \in [\bmone]} u(f_i) - \sum_{i \in [\bmtwo]} u(g_i) \right) \\
          \leq & \exp(- E^2 / 8(\kpa + \bmone - \kmone + \bmtwo - \kmtwo)) \\
          \leq & \delta
    \end{split}
\end{equation*}
To conclude, \Combine\ outperform \Vote\ concerning the testing loss with probability at least $1-\delta$.

Compared to \Vote, denote the additional \pairFeature\ are indexed by $c$, and the additional \singFeature\ in \Vote\ are indexed by $v$.
We have that:
\begin{equation}
    \begin{split}
&\mathbb{P}(\sum_{i \in [c]} (I(f_i(x) > 0) - I(f_i(x) < 0)) - \sum_{j \in [v]} (I(f_j(x) > 0)-I(f_j(x) < 0)) > 0)\\
=& \mathbb{P}(\sum_{i \in [c]} I(f_i(x) > 0) - \sum_{j \in [v]} I(f_j(x) > 0) - \frac{1}{2}[\sum_{i \in [c]} p_i - \sum_{j \in [v]} p_j]> \frac{1}{2} [\sum_{j \in [v]} p_j - \sum_{i \in [c]} p_i] ) \\
\leq& \exp(- (\sum_{j \in [v]} p_j - \sum_{i \in [c]} p_i)^2 / 8|c+v|)
    \end{split}
\end{equation}
Therefore, if $\sum_{j \in [v]} p_j - \sum_{i \in [c]} p_i \geq \sqrt{8 (c+v) \log(1/\delta)}$, the probability is done.
Therefore, for a new data point, \Vote\ can outperforms \Combine\ with high probability.

\end{proof}

\begin{proof}[\textbf{Proof of Lemma~\ref{lem1}}]
We define $r_{[-T]}$ as the features $r_1, \dots, r_{T-1}$.
The proof is divided into two parts, depending on whether $\sum_{i \in [T-1]} \bI(r_i \neq 0)$ is even or odd.
We regard the term $\sum_{i \in [T-1]} \bI(r_i \neq 0)$ as the number of effective features in $r_{[-T]}$.
To simplify the discussion, we rescale $r$ such that $|y r| = 1$ (when $r\neq 0$) or $|y r| = 0 $ (when $r = 0$).

\emph{Case 1}: When the number of effective features in $r_{[-T]}$ is even.
(a. ) If $|\sum_{i \in [T-1]} y r_i| \geq 2$, adding $r_T$ or $r^\prime$ does not alter the predicting probability, namely
{\small
\begin{equation*}
    \begin{split}
        &\bP\left(y \left[r_T + \sum_{i \in [T-1]} y r_i\right] > 0 \ {\bigg|} \ \left|\sum_{i \in [T-1]} y r_i\right| \geq 2\right) + \frac{1}{2} \bP\left(y \left[r_T + \sum_{i \in [T-1]} y r_i\right] = 0 \ {\bigg|} \ \left|\sum_{i \in [T-1]} y r_i\right| \geq 2\right) \\
= &\bP\left(y \left[r^\prime + \sum_{i \in [T-1]} y r_i\right] > 0 \ {\bigg|} \ \left|\sum_{i \in [T-1]} y r_i\right| \geq 2\right) + \frac{1}{2} \bP\left(y \left[r^\prime + \sum_{i \in [T-1]} y r_i\right] = 0 \ {\bigg|} \ \left|\sum_{i \in [T-1]} y r_i\right| \geq 2\right) .
    \end{split}
\end{equation*} }

(b. ) When the number of effective features in $r_{[-T]}$ is even, $|\sum_{i \in [T-1]} y r_i| \neq 1$.

(c. ) When $|\sum_{i \in [T-1]} y r_i| = 0$, due to the assumption that $p(r^\prime) > p(r_T)$ and $\epsilon(r) = p(r)/c$, adding $r^\prime$ helps increase the predicting probability compared to $r_T$, namely
{\small\begin{equation*}
    \begin{split}
        &\bP\left(y \left[r_T + \sum_{i \in [T-1]} y r_i\right] > 0 \ {\bigg|} \ \left|\sum_{i \in [T-1]} y r_i\right| = 0\right) + \frac{1}{2} \bP\left(y \left[r_T + \sum_{i \in [T-1]} y r_i\right] = 0 \ {\bigg|} \ \left|\sum_{i \in [T-1]} y r_i\right| = 0\right) \\
> &\bP\left(y \left[r^\prime + \sum_{i \in [T-1]} y r_i\right] > 0 \ {\bigg|} \ \left|\sum_{i \in [T-1]} y r_i\right| = 0\right) + \frac{1}{2} \bP\left(y \left[r^\prime + \sum_{i \in [T-1]} y r_i\right] = 0 \ {\bigg|} \ \left|\sum_{i \in [T-1]} y r_i\right| = 0\right) .
    \end{split}
\end{equation*}}
The above inequality is derived based on the following equation:
{\small\begin{equation*}
    \begin{split}
          &\bP\left(y \left[r_T + \sum_{i \in [T-1]} y r_i\right] > 0 \ {\bigg|} \ \left|\sum_{i \in [T-1]} y r_i\right| = 0\right) + \frac{1}{2} \bP\left(y \left[r_T + \sum_{i \in [T-1]} y r_i\right] = 0 \ {\bigg|} \ \left|\sum_{i \in [T-1]} y r_i\right| = 0\right) \\
          =&\bP\left(y r_T > 0 \ {\bigg|} \ \left|\sum_{i \in [T-1]} y r_i\right| = 0\right) + \frac{1}{2} \bP\left(y r_T  = 0 \ {\bigg|} \ \left|\sum_{i \in [T-1]} y r_i\right| = 0\right) \\
          =& p(r_T) + \frac{1}{2} \left[1 - p(r_T) - \epsilon(r_T)\right] \\
          =& \frac{1}{2} \left[1 + (1-1/c) p(r_T)\right].
    \end{split}
\end{equation*}}
Since we assume $c>1$, the probability increases with probability $p(r_T)$.

Therefore, under the three conditions, adding $r^\prime$ increase the predicting probability more compared to $r_T$.
In summary, under case 1 (a-c), adding $r^\prime$ increase the predicting probability compared to $r_T$.

\emph{Case 2}: When the number of features in $r_{[-T]}$ is odd. The discussion in (b.) can be a little bit more complex compared to case 1.

(a. ) If $|\sum_{i \in [T-1]} y r_i| \geq 2$, similar to case 1, adding $r_T$ or $r^\prime$ does not alter the predicting probability, namely
{\small\begin{equation*}
    \begin{split}
        &\bP\left(y \left[r_T + \sum_{i \in [T-1]} y r_i\right] > 0 \ {\bigg|} \ \left|\sum_{i \in [T-1]} y r_i\right| \geq 2\right) + \frac{1}{2} \bP\left(y \left[r_T + \sum_{i \in [T-1]} y r_i\right] = 0 \ {\bigg|} \ \left|\sum_{i \in [T-1]} y r_i\right| \geq 2\right) \\
= &\bP\left(y \left[r^\prime + \sum_{i \in [T-1]} y r_i\right] > 0 \ {\bigg|} \ \left|\sum_{i \in [T-1]} y r_i\right| \geq 2\right) + \frac{1}{2} \bP\left(y \left[r^\prime + \sum_{i \in [T-1]} y r_i\right] = 0 \ {\bigg|} \ \left|\sum_{i \in [T-1]} y r_i\right| \geq 2\right) .
    \end{split}
\end{equation*}}

(b. ) If $|\sum_{i \in [T-1]} y r_i| = 1$:
(b.1 ) If $\sum_{i \in [T-1]} y r_i = -1$:
{\small\begin{equation*}
    \begin{split}
        &\bP\left(y \left[r_T + \sum_{i \in [T-1]} y r_i\right] > 0 \ {\bigg|} \ \sum_{i \in [T-1]} y r_i = -1\right) + \frac{1}{2} \bP\left(y \left[r_T + \sum_{i \in [T-1]} y r_i\right] = 0 \ {\bigg|} \ \sum_{i \in [T-1]} y r_i = -1 \right) \\
=& \bP\left(y r_T- 1 > 0 \ {\bigg|} \ \sum_{i \in [T-1]} y r_i = -1\right) + \frac{1}{2} \bP\left(y r_T- 1 = 0 \ {\bigg|} \ \sum_{i \in [T-1]} y r_i = -1 \right) \\
=&  \frac{1}{2} \bP\left(y r_T- 1 = 0 \ {\bigg|} \ \sum_{i \in [T-1]} y r_i = -1 \right)   \\
=& \frac{1}{2} p(r_T).
    \end{split}
\end{equation*}}

(b.2 ) If $\sum_{i \in [T-1]} y r_i = +1$:
{\small\begin{equation*}
    \begin{split}
        &\bP\left(y \left[r_T + \sum_{i \in [T-1]} y r_i\right] > 0 \ {\bigg|} \ \sum_{i \in [T-1]} y r_i = 1\right) + \frac{1}{2} \bP\left(y \left[r_T + \sum_{i \in [T-1]} y r_i\right] = 0 \ {\bigg|} \ \sum_{i \in [T-1]} y r_i = 1 \right) \\
=& \bP\left(y r_T+ 1 > 0 \ {\bigg|} \ \sum_{i \in [T-1]} y r_i = 1\right) + \frac{1}{2} \bP\left(y r_T+ 1 = 0 \ {\bigg|} \ \sum_{i \in [T-1]} y r_i = 1 \right) \\
=& (1 - \epsilon(r_T)) +  \frac{1}{2} \epsilon(r_T) \\
=& 1 - \frac{1}{2c} p(r_T).
    \end{split}
\end{equation*}}

Note that the probability of event (b.1) and the probability of event (b.2) satisfy the following equation by Lemma~\ref{lem:complementary}:
\begin{equation}
    \begin{split}
        \bP\left(\sum_{i \in [T-1]} y r_i = 1 \right) = c \bP\left(\sum_{i \in [T-1]} y r_i = -1 \right).
    \end{split}
\end{equation}

Therefore, the total probability under case (b) is 
\begin{equation*}
    \begin{split}
        &\frac{1}{2} p(r_T) \bP\left(\sum_{i \in [T-1]} y r_i = -1 \right) +  (1 - \frac{1}{2c} p(r_T)) \bP\left(\sum_{i \in [T-1]} y r_i = 1 \right) \\
        =& \bP\left(\sum_{i \in [T-1]} y r_i = 1 \right)
    \end{split}
\end{equation*}
which is independent of $p(r_T)$.
Therefore, adding $r_T$ or $r^\prime$ share the same predicting probability.

(c. ) When the number of effective features in $r_{[-T]}$ is odd, $|\sum_{i \in [T-1]} y r_i| \neq 0$.

In summary, under case 2 (a-c), adding $r^\prime$ do not decrease the predicting probability compared to $r_T$.

The following lemmas are used during the proof.
\begin{lemma}
\label{lem:complementary}
Consider $T-1$ features $r_1, \dots, r_{T-1}$, the following equation holds:
\begin{equation}
    \begin{split}
        \bP\left(\sum_{i \in [T-1]} y r_i = 1 \right) = c \bP\left(\sum_{i \in [T-1]} y r_i = -1 \right).
    \end{split}
\end{equation}
\end{lemma}
\begin{proof}
It can be proved to compare the events $A = \{\sum_{i \in [T-1]} y r_i = 1 \}$ and $B = \{ \sum_{i \in [T-1]} y r_i = -1 \}$.
Every event in $A$ has a complementary event in $B$, namely, 
\begin{equation*}
    \begin{split}
        y r_i = 1 \text{ in } B &\text{ if } y r_i = -1 \text{ in } A\\
        y r_i = -1 \text{ in } B &\text{ if } y r_i = 1 \text{ in } A\\
        y r_i = 0 \text{ in } B &\text{ if } y r_i = 0 \text{ in } A
    \end{split}
\end{equation*}
Comparing each event in $A$ with its complementary event in $B$ leads to the conclusion.
\end{proof}

Combining case 1 and case 2 together leads to the final conclusion.

\end{proof}

\subsection{Generalize Theorem~\ref{thm:combineoverfitting} to more modalities}
\label{generalize_multi_modal}

We next show that the results in Theorem~\ref{thm:combineoverfitting} can be generalized to the regime of more modals. 

Specifically, we assume a $T$-modal regimes, and denote the modals as $x^{m_i}, i \in [T]$.
In uni-modal pre-training approaches, let $b_{m_i}$ denote the number of returned features in modal $i$.
In multi-modal joint training, let $k_{m_i}$ denote the number of uni-modal features for modal $i$, and $k_{pa}$ denote the number of returned paired features. 
We derive the following Theorem~\ref{thm:modalcombineoverfitting} for the multi-modal regimes.

\begin{restatable}{theorem}{modalCombineoverfitting}
\label{thm:modalcombineoverfitting}

Based on the above notations, we provide three types of  laziness from three perspectives:

\begin{enumerate}
    \item[(a. )] \textbf{Quantity Laziness}: $\sum_i k_{m_i} + k_{pa} \leq \min_i \{b_{m_i}\}$.
    \item[(b. )] \textbf{Uni-modal Laziness}: Each modality in \Combine\ performs worse than uni-modal training.
    \item[(c. )] \textbf{Performance Laziness}:
    Consider a new testing point, then for every $\delta>0$, if the following inequality holds:
    $$\sum_{i \in [\kpa]} p(h_i) \leq \sum_{i \in [\min_i\{ b_{m_i} \}+1, \sum_i b_{m_i}]} p_{[i]} + \Delta(\delta),
    $$
    where $\Delta(\delta) = \sqrt{ 8 (\kpa + \sum_{j} [b_{m_j} - k_{m_j}]) \log(1/\delta)}$, then with probability\footnote{The probability is taken over the randomness of the testing point} at least $1-\delta$, \Vote\ outperform \Combine\ concerning the loss on the testing point with probability.
\end{enumerate}

\end{restatable}

\subsection{Proof of Theorem~\ref{thm:umt}}
\label{append:proof_umt}


\umt*

\begin{proof}
The core of Theorem~\ref{thm:umt} is to clarify the training priority.
We revisit the notations of Theorem~\ref{thm:combineoverfitting} as follows without further clarification.
At the end of the training, \Vote\ learn $\bmone + \bmtwo$ useful features,  
namely, $f_1, \dots, f_{\bmone}, g_1, \dots, g_{\bmtwo}$.
And \Combine\ learn $\kmone + \kmtwo + \kpa $ features: $f_1, \dots, f_{\kmone}, g_1, \dots, g_{\kmtwo}$, $h_1, \dots, h_{\kpa}$.
We note that there are still many empty features $e_i$ in the model due to the initialization.

By distillation, the model learns the features according to the new priority.
Since the set $\cS$ is not empty, there exists \pairFeature s\ that is learned before the empty features.
By distillation, the model would learn all the useful features that appear in \Uni, as well as those features in set $\cS$.
Therefore, UMT outperforms \Vote when there exists useful paired features.


\end{proof}

\subsection{A concrete example to illustrate Theorem~\ref{thm:combineoverfitting}}
\label{append:example}

We next provide a concrete example to better illustrate the Modality Laziness issues.
For Example~\ref{eg:1}, we aim to show the Modality Laziness issues.
For Example~\ref{eg:2}, we aim to show the role of the pushing force.

\begin{example}
\label{eg:1}
Consider modality $\modalOne$ with features $f_1, f_2, f_3$ (corresponding prediction probability $p = 0.2, 0.1, 0.05$), and modality $\modalTwo$ with features $g_1, g_2, g_3$ (corresponding prediction probability $p = 0.15, 0.08, 0.02$).
We show the dataset in Table~\ref{tab:dataset} and aim to minimize the training loss to zero.

In \Uni, we learn features $f_1$, $f_2$ and $f_3$ on modality $\modalOne$ (similarly, $g_1$, $g_2$, and $g_3$ on modality $\modalTwo$).
Therefore, we learn features $f_1, f_2, f_3, g_1, g_2, g_3$ in \Vote.
In \Combine\ without \pairFeature, we can only learn three features $f_1, f_2, g_2$ due to the training priority $f_1>g_1>f_2>g_2>f_3>g_3$ (decreasing order in $p$).
This phenomenon is caused by modality laziness.

We next consider another \pairFeature\ $h$ with probability $p = 0.28$.
Under the case, \Combine\ only learn two features $h$ and $f_1$.
Therefore, when $h$ is not powerful enough, \Vote\ outperforms \Combine. 

\begin{table*}[t]
    \centering
    \caption{Dataset used in Example~\ref{eg:1}. $+$ means the feature is larger than zero and $-$ means the feature is less than zero. 
    We denote the predicting probability by $p$ and the rectified probability (due to pushing force) by $p^\prime$.
    }
    \begin{tabular}{c|c|c|c|c|c|c|c|c}
    \ & $f_1$&$f_2$&$f_3$&$g_1$&$g_2$&$g_3$&$h$&y\\
    \hline
    $p$& 0.20&0.10&0.05&0.15&0.08&0.02& 0.28 & / \\
    $p^\prime$ & 0.35($\uparrow$)&0.25($\uparrow$)&0.20($\uparrow$)&0.32($\uparrow$)&0.23($\uparrow$)&0.17($\uparrow$)& 0.28 & /  \\
       \hline
        data a &+&+&+&+&-&+&+&+1\\
        data b &0&+&0&+&+&-&+&+1\\
        data c &+&+&0&-&+&+&0&-1\\
        data d &+&-&+&+&+&0&+&-1\\
        \hline
    \end{tabular}
    
    \label{tab:dataset}
\end{table*}

\end{example}

\begin{example}
\label{eg:2}
We follow the notations and dataset in Example~\ref{eg:1}.
By applying the pushing force, assume that each probability of \singFeature\ boosts 0.15, which changes the training priority to $f_1 > g_1 > h > f_2 > g_2 > f_3 > g_3$ (decreasing order in $p^\prime$).
Therefore, \Combine\ (with pushing force) learns $f_1, f_2, h$.
As a comparison, \Combine\ (without pushing force) can only learn $f_1, h$. Therefore, pushing force helps learn more features.
We additionally remark that we only consider the training error in this example, and there might be other penalties in practice (\textit{e.g.}, distillation loss).
\end{example}

\end{document}